\documentclass{article}

\usepackage{microtype}
\usepackage{graphicx}
\usepackage{subfigure}
\usepackage{natbib}
\usepackage{adjustbox}
\usepackage{multirow}
\usepackage[utf8]{inputenc} 
\usepackage[T1]{fontenc}    
\usepackage{url}            
\usepackage{booktabs}       
\usepackage{amsfonts}       
\usepackage{nicefrac}       
\usepackage{microtype}      
\usepackage{amsthm,amssymb}
\usepackage[dvipsnames,hyperref]{xcolor}
\usepackage{xifthen}

\usepackage{tikz}
\usepackage{standalone}
\usepackage{pgfplots}
\pgfplotsset{compat=1.17}

\usepackage{amsmath}
\usetikzlibrary{positioning}
\usetikzlibrary{arrows}
\usetikzlibrary{calc}
\usetikzlibrary{shapes.misc}
\usetikzlibrary{shapes.geometric}
\usetikzlibrary{decorations.pathmorphing}
\usetikzlibrary{fit}

\definecolor{flat-red}{RGB}{231, 76, 60}
\definecolor{flat-blue}{RGB}{41, 128, 185}

\usepackage{amsmath,amsfonts,bm}

\def\eqref#1{equation~\ref{#1}}

\def\1{\bm{1}}

\def\vzero{{\bm{0}}}

\def\va{{\bm{a}}}
\def\vb{{\bm{b}}}

\def\vg{{\bm{g}}}

\def\vm{{\bm{m}}}

\def\vv{{\bm{v}}}

\def\vx{{\bm{x}}}
\def\vy{{\bm{y}}}
\def\vz{{\bm{z}}}

\def\mA{{\bm{A}}}

\def\mD{{\bm{D}}}

\def\mI{{\bm{I}}}

\DeclareMathAlphabet{\mathsfit}{\encodingdefault}{\sfdefault}{m}{sl}
\SetMathAlphabet{\mathsfit}{bold}{\encodingdefault}{\sfdefault}{bx}{n}

\def\gB{{\mathcal{B}}}

\def\gG{{\mathcal{G}}}

\def\sR{{\mathbb{R}}}

\newcommand{\E}{\mathbb{E}}

\DeclareMathOperator*{\argmin}{argmin}

\newcommand{\eqcomment}[2]{\textrm{\tikz[baseline]\node[font=\footnotesize,inner sep=0pt,yshift=2pt,opacity=0.3] {(#1)};}\ifthenelse{\isempty{#2}}{\qquad}{#2}}

\newtheorem{theorem}{Theorem}
\newtheorem{corollary}{Corollary}
\newtheorem{proposition}{Proposition}
\newtheorem{lemma}{Lemma}
\newtheorem{assumption}{Assumption}
\newtheorem{definition}{Definition}

\DeclareMathOperator{\prox}{prox}
\DeclareMathOperator{\diag}{diag}
\DeclareMathOperator{\dist}{dist}
\DeclareMathOperator{\MCP}{MCP}

\usepackage{hyperref}
\usepackage[capitalize]{cleveref}

\crefformat{section}{#2Section~#1#3}
\crefformat{appendix}{#2Appendix~#1#3}

\crefformat{equation}{(#2#1#3)}
\crefmultiformat{equation}{#2Equations~#1#3}%
{ \&~#2#1#3}{, #2#1#3}{, and~(#2#1#3)}

\crefformat{table}{#2Table~#1#3}
\crefformat{figure}{#2Figure~#1#3}

\crefformat{theorem}{#2Theorem~#1#3}
\crefmultiformat{theorem}{#2Theorems~#1#3}%
{ \&~#2#1#3}{, #2#1#3}{, and~(#2#1#3)}

\crefformat{lemma}{#2Lemma~#1#3}
\crefformat{proposition}{#2Proposition~#1#3}

\crefformat{algorithm}{#2Algorithm~#1#3}
\crefmultiformat{algorithm}{#2Algorithms~#1#3}%
{ \&~#2#1#3}{, #2#1#3}{, and~(#2#1#3)}

\crefformat{corollary}{#2Corollary~#1#3}
\crefformat{definition}{#2Definition~#1#3}

\crefformat{assumption}{#2Assumption~#1#3}
\crefmultiformat{assumption}{#2Assumptions~#1#3}%
{ \&~#2#1#3}{, #2#1#3}{, and~(#2#1#3)}

\usepackage[accepted]{icml2021}

\icmltitlerunning{Structured Sparsity Inducing Adaptive Optimizers for Deep Learning}

\begin{document}

\twocolumn[
\icmltitle{Structured Sparsity Inducing Adaptive Optimizers for Deep Learning}



\icmlsetsymbol{equal}{*}

\begin{icmlauthorlist}
\icmlauthor{Tristan Deleu}{mila,udem}
\icmlauthor{Yoshua Bengio}{mila,udem,cifar_ai_chair,cifar_fellow}
\end{icmlauthorlist}

\icmlaffiliation{mila}{Mila}
\icmlaffiliation{udem}{Universite de Montreal}
\icmlaffiliation{cifar_ai_chair}{CIFAR AI Chair}
\icmlaffiliation{cifar_fellow}{CIFAR Senior Fellow}

\icmlcorrespondingauthor{Tristan Deleu}{deleutri@mila.quebec}

\icmlkeywords{}

\vskip 0.3in
]



\printAffiliationsAndNotice{}  

\begin{abstract}
The parameters of a neural network are naturally organized in groups, some of which might not contribute to its overall performance. To prune out unimportant groups of parameters, we can include some non-differentiable penalty to the objective function, and minimize it using proximal gradient methods. In this paper, we derive the weighted proximal operator, which is a necessary component of these proximal methods, of two structured sparsity inducing penalties. Moreover, they can be approximated efficiently with a numerical solver, and despite this approximation, we prove that existing convergence guarantees are preserved when these operators are integrated as part of a generic adaptive proximal method. Finally, we show that this adaptive method, together with the weighted proximal operators derived here, is indeed capable of finding solutions with structure in their sparsity patterns, on representative examples from computer vision and natural language processing.
\end{abstract}

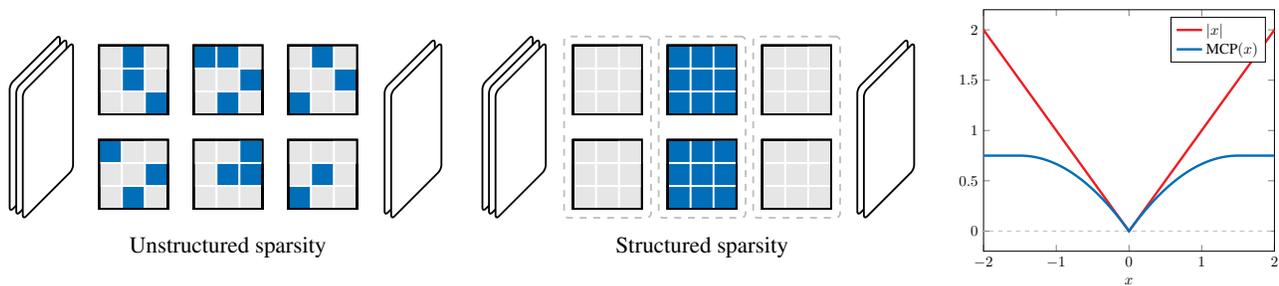
\begin{figure*}[t]
    \centering
    \begin{tikzpicture}[x=15pt, y=15pt, filter/.style={inner sep=0pt}]
    \def\zerocolor{NavyBlue}
    \def\onecolor{gray!20}
    \def\innerbordercolor{white}
    \def\outerbordercolor{black}
    
    \newcommand{\networkLayer}[6]{
    \def\a{#1} 
    \def\b{#2}
	\def\t{#3} 
	\def\d{#4} 
			
	\fill[#5,rounded corners=2pt] (\t,0,\d) -- (\t,\a,\d) -- (\t,\a,\b+\d) -- (\t,0,\b+\d) -- cycle;

	\draw[#6,rounded corners=2pt] (\t,0,\d) -- (\t,\a,\d) -- (\t,\a,\b+\d) -- (\t,0,\b+\d) -- cycle;
	}

    \node[] (unstructured) at (0, 0) {\begin{tikzpicture}
    \networkLayer{3.7}{2}{0.0}{-0.5}{color=white}{thick};
    \networkLayer{3.7}{2}{0.3}{-0.35}{color=white}{thick};
    \networkLayer{3.7}{2}{0.6}{-0.2}{color=white}{thick};
    
    \networkLayer{3.7}{2}{11.0}{-0.35}{color=white}{thick};
    \networkLayer{3.7}{2}{11.3}{-0.2}{color=white}{thick};

    \node[filter] (filter11) at (2.5, 0) {\begin{tikzpicture}[x=10pt,y=10pt]
    \foreach \v [count=\i from 0, evaluate=\i as \x using {Mod(\i,3)}, evaluate=\i as \y using floor(\i/3)] in {1,0,1,1,1,0,0,1,1} {
        \ifthenelse{\v=1}{\draw[fill=\onecolor,draw=\innerbordercolor,thick] (\x,\y) +(-0.5,-0.5) rectangle ++(0.5,0.5);}{\draw[fill=\zerocolor,draw=\innerbordercolor,thick] (\x,\y) +(-0.5,-0.5) rectangle ++(0.5,0.5);}
        \draw[thick,\outerbordercolor] (-0.5,-0.5) rectangle (2.5,2.5);
    }
    \end{tikzpicture}};
    
    \node[filter, above=1em of filter11] (filter21) {\begin{tikzpicture}[x=10pt,y=10pt]
    \foreach \v [count=\i from 0, evaluate=\i as \x using {Mod(\i,3)}, evaluate=\i as \y using floor(\i/3)] in {1,1,0,1,0,1,1,0,1} {
        \ifthenelse{\v=1}{\draw[fill=\onecolor,draw=\innerbordercolor,thick] (\x,\y) +(-0.5,-0.5) rectangle ++(0.5,0.5);}{\draw[fill=\zerocolor,draw=\innerbordercolor,thick] (\x,\y) +(-0.5,-0.5) rectangle ++(0.5,0.5);}
        \draw[thick,\outerbordercolor] (-0.5,-0.5) rectangle (2.5,2.5);
    }
    \end{tikzpicture}};
    
    \node[filter, right=1em of filter11] (filter12) {\begin{tikzpicture}[x=10pt,y=10pt]
    \foreach \v [count=\i from 0, evaluate=\i as \x using {Mod(\i,3)}, evaluate=\i as \y using floor(\i/3)] in {1,1,1,1,0,0,1,1,0} {
        \ifthenelse{\v=1}{\draw[fill=\onecolor,draw=\innerbordercolor,thick] (\x,\y) +(-0.5,-0.5) rectangle ++(0.5,0.5);}{\draw[fill=\zerocolor,draw=\innerbordercolor,thick] (\x,\y) +(-0.5,-0.5) rectangle ++(0.5,0.5);}
        \draw[thick,\outerbordercolor] (-0.5,-0.5) rectangle (2.5,2.5);
    }
    \end{tikzpicture}};
    
    \node[filter, right=1em of filter21] (filter22) {\begin{tikzpicture}[x=10pt,y=10pt]
    \foreach \v [count=\i from 0, evaluate=\i as \x using {Mod(\i,3)}, evaluate=\i as \y using floor(\i/3)] in {1,0,1,1,1,0,0,0,1} {
        \ifthenelse{\v=1}{\draw[fill=\onecolor,draw=\innerbordercolor,thick] (\x,\y) +(-0.5,-0.5) rectangle ++(0.5,0.5);}{\draw[fill=\zerocolor,draw=\innerbordercolor,thick] (\x,\y) +(-0.5,-0.5) rectangle ++(0.5,0.5);}
        \draw[thick,\outerbordercolor] (-0.5,-0.5) rectangle (2.5,2.5);
    }
    \end{tikzpicture}};
    
    \node[filter, right=1em of filter12] (filter13) {\begin{tikzpicture}[x=10pt,y=10pt]
    \foreach \v [count=\i from 0, evaluate=\i as \x using {Mod(\i,3)}, evaluate=\i as \y using floor(\i/3)] in {0,1,1,1,0,1,1,1,1} {
        \ifthenelse{\v=1}{\draw[fill=\onecolor,draw=\innerbordercolor,thick] (\x,\y) +(-0.5,-0.5) rectangle ++(0.5,0.5);}{\draw[fill=\zerocolor,draw=\innerbordercolor,thick] (\x,\y) +(-0.5,-0.5) rectangle ++(0.5,0.5);}
        \draw[thick,\outerbordercolor] (-0.5,-0.5) rectangle (2.5,2.5);
    }
    \end{tikzpicture}};
    
    \node[filter, right=1em of filter22] (filter23) {\begin{tikzpicture}[x=10pt,y=10pt]
    \foreach \v [count=\i from 0, evaluate=\i as \x using {Mod(\i,3)}, evaluate=\i as \y using floor(\i/3)] in {0,1,1,1,1,0,1,0,1} {
        \ifthenelse{\v=1}{\draw[fill=\onecolor,draw=\innerbordercolor,thick] (\x,\y) +(-0.5,-0.5) rectangle ++(0.5,0.5);}{\draw[fill=\zerocolor,draw=\innerbordercolor,thick] (\x,\y) +(-0.5,-0.5) rectangle ++(0.5,0.5);}
        \draw[thick,\outerbordercolor] (-0.5,-0.5) rectangle (2.5,2.5);
    }
    \end{tikzpicture}};
    \node[below=0.8em of filter12] {Unstructured sparsity};
    \end{tikzpicture}};
    
    \node[right=1em of unstructured] (structured) {\begin{tikzpicture}[group/.style={thick, draw=gray!50, dashed, rounded corners=2pt}]
    \networkLayer{3.7}{2}{0.0}{-0.5}{color=white}{thick};
    \networkLayer{3.7}{2}{0.3}{-0.35}{color=white}{thick};
    \networkLayer{3.7}{2}{0.6}{-0.2}{color=white}{thick};
    
    \networkLayer{3.7}{2}{11.0}{-0.35}{color=white}{thick};
    \networkLayer{3.7}{2}{11.3}{-0.2}{color=white}{thick};

    \node[filter] (filter11) at (1.5, 0) {\begin{tikzpicture}[x=10pt,y=10pt]
    \foreach \v [count=\i from 0, evaluate=\i as \x using {Mod(\i,3)}, evaluate=\i as \y using floor(\i/3)] in {1,1,1,1,1,1,1,1,1} {
        \ifthenelse{\v=1}{\draw[fill=\onecolor,draw=\innerbordercolor,thick] (\x,\y) +(-0.5,-0.5) rectangle ++(0.5,0.5);}{\draw[fill=\zerocolor,draw=\innerbordercolor,thick] (\x,\y) +(-0.5,-0.5) rectangle ++(0.5,0.5);}
        \draw[thick,\outerbordercolor] (-0.5,-0.5) rectangle (2.5,2.5);
    }
    \end{tikzpicture}};
    
    \node[filter, above=1em of filter11] (filter21) {\begin{tikzpicture}[x=10pt,y=10pt]
    \foreach \v [count=\i from 0, evaluate=\i as \x using {Mod(\i,3)}, evaluate=\i as \y using floor(\i/3)] in {1,1,1,1,1,1,1,1,1} {
        \ifthenelse{\v=1}{\draw[fill=\onecolor,draw=\innerbordercolor,thick] (\x,\y) +(-0.5,-0.5) rectangle ++(0.5,0.5);}{\draw[fill=\zerocolor,draw=\innerbordercolor,thick] (\x,\y) +(-0.5,-0.5) rectangle ++(0.5,0.5);}
        \draw[thick,\outerbordercolor] (-0.5,-0.5) rectangle (2.5,2.5);
    }
    \end{tikzpicture}};
    
    \node[filter, right=1em of filter11] (filter12) {\begin{tikzpicture}[x=10pt,y=10pt]
    \foreach \v [count=\i from 0, evaluate=\i as \x using {Mod(\i,3)}, evaluate=\i as \y using floor(\i/3)] in {0,0,0,0,0,0,0,0,0} {
        \ifthenelse{\v=1}{\draw[fill=\onecolor,draw=\innerbordercolor,thick] (\x,\y) +(-0.5,-0.5) rectangle ++(0.5,0.5);}{\draw[fill=\zerocolor,draw=\innerbordercolor,thick] (\x,\y) +(-0.5,-0.5) rectangle ++(0.5,0.5);}
        \draw[thick,\outerbordercolor] (-0.5,-0.5) rectangle (2.5,2.5);
    }
    \end{tikzpicture}};
    
    \node[filter, right=1em of filter21] (filter22) {\begin{tikzpicture}[x=10pt,y=10pt]
    \foreach \v [count=\i from 0, evaluate=\i as \x using {Mod(\i,3)}, evaluate=\i as \y using floor(\i/3)] in {0,0,0,0,0,0,0,0,0} {
        \ifthenelse{\v=1}{\draw[fill=\onecolor,draw=\innerbordercolor,thick] (\x,\y) +(-0.5,-0.5) rectangle ++(0.5,0.5);}{\draw[fill=\zerocolor,draw=\innerbordercolor,thick] (\x,\y) +(-0.5,-0.5) rectangle ++(0.5,0.5);}
        \draw[thick,\outerbordercolor] (-0.5,-0.5) rectangle (2.5,2.5);
    }
    \end{tikzpicture}};
    
    \node[filter, right=1em of filter12] (filter13) {\begin{tikzpicture}[x=10pt,y=10pt]
    \foreach \v [count=\i from 0, evaluate=\i as \x using {Mod(\i,3)}, evaluate=\i as \y using floor(\i/3)] in {1,1,1,1,1,1,1,1,1} {
        \ifthenelse{\v=1}{\draw[fill=\onecolor,draw=\innerbordercolor,thick] (\x,\y) +(-0.5,-0.5) rectangle ++(0.5,0.5);}{\draw[fill=\zerocolor,draw=\innerbordercolor,thick] (\x,\y) +(-0.5,-0.5) rectangle ++(0.5,0.5);}
        \draw[thick,\outerbordercolor] (-0.5,-0.5) rectangle (2.5,2.5);
    }
    \end{tikzpicture}};
    
    \node[filter, right=1em of filter22] (filter23) {\begin{tikzpicture}[x=10pt,y=10pt]
    \foreach \v [count=\i from 0, evaluate=\i as \x using {Mod(\i,3)}, evaluate=\i as \y using floor(\i/3)] in {1,1,1,1,1,1,1,1,1} {
        \ifthenelse{\v=1}{\draw[fill=\onecolor,draw=\innerbordercolor,thick] (\x,\y) +(-0.5,-0.5) rectangle ++(0.5,0.5);}{\draw[fill=\zerocolor,draw=\innerbordercolor,thick] (\x,\y) +(-0.5,-0.5) rectangle ++(0.5,0.5);}
        \draw[thick,\outerbordercolor] (-0.5,-0.5) rectangle (2.5,2.5);
    }
    \end{tikzpicture}};
    
    \node[group, fit=(filter11) (filter21)] {};
    \node[group, fit=(filter12) (filter22)] {};
    \node[group, fit=(filter13) (filter23)] {};
    
    \node[below=0.8em of filter12] {Structured sparsity};
    \end{tikzpicture}};
    
    \node[right=1em of structured] (ssi-penalties) {\begin{tikzpicture}[scale=0.65,
declare function={
mcp(\x) = (abs(\x)>=\BETA*\LAMBDA) * (0.5 * \BETA * \LAMBDA^2) + (abs(\x)<\BETA*\LAMBDA) * (\LAMBDA * abs(\x) - 0.5 * \x^2 / \BETA);
}]
\begin{axis}[xmin=-2, xmax=2, xlabel={$x$}, legend entries={$|x|$,$\textrm{MCP}(x)$}, legend pos=north east, legend cell align=left]
\newcommand\BETA{1.5}
\newcommand\LAMBDA{1}

\addplot [mark=none, gray!50, dashed, forget plot] {0};
\addplot [Red, domain=-2:2, samples=201, ultra thick] {abs(x)};
\addplot [NavyBlue, domain=-2:2, samples=201, ultra thick] {mcp(x)};
\end{axis}
\end{tikzpicture}};
    
\end{tikzpicture}
    \vspace{-2em}
    \caption{Structured sparsity in neural networks. (Left) An example of the effect of unstructured sparsity inducing penalties, e.g. the $\ell_{1}$ penalty, on a convolutional layer. The zeros of each filter are represented in blue. (Middle) An example of the effect of structured sparsity inducing penalties, e.g. the mixed $\ell_{1}/\ell_{2}$ norm, on a similar convolutional layer. The filters are grouped together, depending on their input channel, and whole groups are encouraged to be equal to zero: this is a \emph{channel-wise structure} \citep{wen2016structured}. Here, the sparsity pattern effectively removes the effect of the second input channel from the network's predictions. (Right) Graphs of the two structured sparsity inducing penalties (in 1-D, for visualization) studied in this paper. The $\ell_{1}$ norm applies a larger penalty for large values of $x$, whereas the Minimax Concave Penalty (MCP) saturates after a certain point.}
    \label{fig:structured-sparsity}
\end{figure*}

\section{Introduction}
\label{sec:introduction}
In machine learning, a large majority of problems involve the minimization of a composite loss function of the form
\begin{equation}
    \min_{\vx\in\sR^{N}} f(\vx) + h(\vx),
    \label{eq:composite-loss}
\end{equation}
where $f$ is a differentiable loss function, and $h$ is a penalty function (or \emph{regularizer}). Some standard penalties used in the context of training neural networks include the squared $\ell_{2}$ norm (also called \emph{ridge regression} in statistics, or \emph{weight decay} in machine learning) and elastic net \citep{zou2005elasticnet}. In this work, we focus on non-differentiable, and possibly non-convex, penalties leading to sparse solutions of this optimization problem. In particular, we are interested in penalties which encourage structure in the sparsity patterns of these solutions \citep{kyrillidis2015structuredsparsity}, such as the mixed $\ell_{1}/\ell_{2}$ norm, also known as \emph{group Lasso} when $f$ is the squared loss \citep{yuan2006grouplasso}. These penalties tend to produce solutions where complete groups of variables are zeroed-out, as opposed to the more frequently used $\ell_{1}$ norm which yields unstructured sparsity. This class of penalties is called \emph{structured sparsity inducing} penalties \citep{bach2011optimization,bach2012structuredconvexopt}.

This family of regularizers is particularly appropriate for neural networks, where the parameters are naturally organized in layers and groups of filters \citep{krizhevsky2012alexnet,ioannou2017deeproots}. If some of these groups of parameters were exactly equal to zero, this would effectively be equivalent to deactivating certain neurons in the network \citep{zhou2016lessismore}. As a consequence, it may have an impact on the size of the model, reducing its memory footprint and computational cost at inference time. However in these structured sparsity inducing penalties, the groups of variables can be arbitrary, and they do not necessarily have to follow the layered nature of the neural network if we have some prior knowledge about the structure of the problem, such as a causal structure \citep{germain2015made,lachapelle2020grandag,ke2019neuralcausal}, or for feature selection.

To solve an optimization problem like the one in \cref{eq:composite-loss}, it is standard practice in Deep Learning to use first-order methods, and in particular adaptive methods such as Adagrad \citep{duchi2011adagrad}, RMSprop \citep{tieleman2012rmsprop}, or Adam \citep{kingma2014adam}. Even when the penalty function $h$ is non-differentiable, as is the case for structured sparsity inducing penalties, it is still possible to optimize this composite loss function using stochastic subgradient methods \citep{wen2016structured}. But even if these penalties are meant to introduce sparsity, the solutions found by subgradient methods are typically not sparse \citep{bach2011optimization}, and they often require a post-processing step where the parameters with a small magnitude are pruned out. An alternative approach is to use \emph{proximal gradient methods}, which are specifically designed for the optimization of composite loss functions. These methods avoid choosing an arbitrary subgradient of $h$ at non-differentiable points by applying a \emph{proximal operator} \citep{moreau1962proximal}, which can often be efficiently computed, to a standard gradient descent update rule.
\newpage
While proximal gradient methods are often presented as derived from vanilla gradient descent, it has been shown that they can also be applied to adaptive first-order methods \citep{yang2020proxsgd,yun2020proxgen}, using a weighted version of the proximal operator \citep{hiriart1993convex}. For many penalties, including the $\ell_{1}$ penalty \citep{becker2019quasinewton}, the weighted proximal operator is a straightforward extension of its unweighted counterpart, and can often be computed in closed-form. However for structured sparsity inducing penalties, computing their weighted proximal operator becomes more challenging than their unweighted equivalent, due to their grouping property.

In this paper, we derive the \emph{weighted proximal operator} of two different structured sparsity inducing penalties: the mixed $\ell_{1}/\ell_{2}$ norm, and the group Minimax Concave Penalty \citep[group MCP;][]{breheny2009groupmcp}. We show that unlike their (unweighted) proximal operator, no closed-form solution exists for the weighted proximal operator of these penalties, although they can be approximated efficiently using a root finding algorithm. When this approximation is integrated as part of a stochastic proximal gradient method \citep{yun2020proxgen}, we prove that the resulting inexact method retains the same convergence guarantees as if the exact weighted proximal operator was applied. We further support empirically that this algorithm can find solutions with structure in their sparsity patterns, with examples from computer vision and natural language processing.

\section{Related work}
\label{sec:related-work}
\paragraph{Pruning} In addition to having a smaller memory footprint, pruning groups of variables that only have a small impact on the network's performance, such as filters in a convolutional neural network, can also significantly speedup inference on standard architectures \citep{li2017pruning}. Closely related to our work, \citet{wen2016structured} used the mixed $\ell_{1}/\ell_{2}$ norm in order to identify these groups, and minimized the objective with subgradient methods. However this required a post-processing step at the end of training because subgradient methods may fail at finding truly sparse solutions \citep{bach2011optimization}. In contrast, proximal gradient methods tend to not suffer from this problem, and they can find solutions with exact zeros. \citet{lin2019toward} also proposed to use structured sparsity regularization for pruning, with an approach inspired by ADMM.

\paragraph{Proximal gradient methods} Proximal gradient methods (also known as \emph{forward-backward splitting} methods; \citealp{singer2009fobos}) is a class of first-order methods that are particularly designed for the optimization of composite objectives of the form of \cref{eq:composite-loss}, where the non-differentiable penalty $h$ admits a proximal operator that can be efficiently computed. See \citet{parikh2014proximal} for a general introduction to proximal algorithms, and \citet{bach2011optimization} for an overview of proximal gradient methods for sparsity-inducing penalties (including structured sparsity).

Proximal gradient methods have also been embedded into adaptive first-order methods, such as Adam or RMSProp \citep{lee2019proxadam}. \citet{yang2020proxsgd} introduced \textsc{Prox-SGD}, an algorithm based on a weighted proximal operator to correct for the rescaling introduced by these adaptive methods. This algorithm was further refined in \textsc{ProxGen} \citep{yun2020proxgen}, to treat the update as an exact proximal gradient step. Both of these methods are generic, and they can be applied to any penalty, provided the weighted proximal operator can be computed efficiently. But so far, their applications have been limited to sparsity-inducing penalties such as $\ell_{q}^{q}$, with $q \in [0, 1]$, whose weighted proximal operators can be obtained in closed-form thanks to the separability of these penalties \citep{becker2019quasinewton}. Our work is complementary, since we give an efficient algorithm to compute the weighted proximal operator of some structured sparsity inducing penalties, for which there is no closed-form solution available. \citet{melchior2019adaprox} also propose an alternative algorithm where the weighted proximal operator is approximated by iterating over the (unweighted) proximal operator. \citet{chen2020half} proposed a Half-Space step, in conjunction with projected gradient methods, to increase the level of group-sparsity.

Finally, although we are focusing our attention here on first-order methods, which are more appropriate for Deep Learning applications, it is important to note that there also exists proximal quasi-Newton methods \citep{becker2019quasinewton}, also based on weighted proximal operators, which are adapted from limited-memory quasi-Newton methods such as SR1 and L-BFGS. 

\section{Structured sparsity inducing penalties}
\label{sec:structured-sparsity-penalties}
Throughout this paper, we assume that we have a collection $\gG$ of (disjoint) groups of variables; for example, these groups might correspond to individual filters in a convolutional neural network \citep{wen2016structured}. Structured sparsity inducing penalties are penalty functions that encourage whole groups of variables as defined by $\gG$ to be ignored (see \cref{fig:structured-sparsity} for an illustration). In this section, we will first recall the definition of two of these penalties, namely the mixed $\ell_{1}/\ell_{2}$ norm and group MCP, and how they have been used in the past in conjunction with proximal gradient descent.

\subsection{Mixed $\ell_{1}/\ell_{2}$ norm}
\label{sec:mixed-l1-l2-norm}
Taking inspiration from the $\ell_{1}$ penalty, which is widely adopted for unstructured sparsity, the mixed $\ell_{1}/\ell_{2}$ norm can be thought of as the $\ell_{1}$ norm of a vector consisting of the $\ell_{2}$ norms over each group in $\gG$. More precisely, this penalty is defined as
\begin{equation}
    h(\vx) \triangleq \sum_{g\in\gG} \lambda_{g} \|\vx_{g}\|_{2},
    \label{eq:l1-l2}
\end{equation}
where $\|\vx_{g}\|_{2} = \sqrt{\sum_{j\in g}x_{j}^{2}}$ is the $\ell_{2}$ norm of the vector $\vx$, restricted to group $g$. To account for groups of varying sizes, the weighting $\lambda_{g}$ is typically set to $\lambda_{g} = \lambda\sqrt{|g|}$, where $|g|$ is the number of elements in group $g$ \citep{murphy2012mlapp}.

\subsection{Bias reduction with non-convex penalties}
\label{sec:non-convex-penalties}
Although applying the $\ell_{1}$ penalty may result in sparse solutions, this penalty is also known to suffer from a \emph{shrinkage} effect \citep{copas1983shrinkage}, introducing bias in the model. The mixed $\ell_{1}/\ell_{2}$ norm inherits from this bias, but this time at the group level: groups of variables get pushed invariably towards $0$, regardless of their $\ell_{2}$ norm.

In order to reduce this bias, \citet{zhang2010mcp} introduced a non-convex penalty called the \emph{Minimax Concave Penalty} (MCP). MCP operates under two regimes: similarly to the $\ell_{1}$ penalty, small values are shrunk towards $0$, but contrary to $\ell_{1}$, sufficiently large values are not penalized as much anymore. Formally, this penalty function is defined as
\begin{equation}
    \MCP(x;\beta, \lambda) \triangleq \left\{\begin{array}{cl}
        \lambda |x| - \dfrac{x^{2}}{2\beta} & \textrm{if $|x| \leq \beta\lambda$} \\
        \dfrac{\beta\lambda^{2}}{2} & \textrm{otherwise,}
    \end{array}\right.
    \label{eq:mcp}
\end{equation}
with $\beta > 1$ controlling the magnitude beyond which values are no longer penalized. \cref{fig:structured-sparsity} shows a comparison of $\MCP$ against the $\ell_{1}$ penalty. Similar to the mixed $\ell_{1}/\ell_{2}$ norm, it is also possible to apply MCP to the $\ell_{2}$ norm of groups of variables, in order to encourage structured sparsity while retaining the benefits of MCP: this is called group MCP \citep{breheny2009groupmcp}
\begin{equation}
    h(\vx) = \sum_{g\in\gG}\MCP(\|\vx_{g}\|_{2};\beta, \lambda_{g}).
    \label{eq:mcp-l2}
\end{equation}

\subsection{Proximal gradient descent}
\label{sec:proximal-gradient-descent}
In order to minimize the composite objective in \cref{eq:composite-loss}, the proximal gradient descent algorithm updates an iterate $\vx_{t}$ with an update similar to standard gradient descent, based on the gradient of the differentiable part $f$
\begin{equation*}
    \vx_{t+1} \leftarrow \prox_{\alpha h}\!\big(\vx_{t} - \alpha \nabla_{\vx}f(\vx_{t})\big),
    \label{eq:pgd-update}
\end{equation*}
where $\alpha$ is the learning rate, and $\prox_{\alpha h}$ is the \emph{proximal operator} \citep{moreau1962proximal} of the penalty $h$, defined as
\begin{equation}
    \prox_{\alpha h}(\vx) = \argmin_{\vz\in\sR^{n}}\frac{1}{2}\|\vz - \vx\|_{2}^{2} + \alpha h(\vz).
    \label{eq:proximal-operator}
\end{equation}
Interestingly, proximal gradient descent generalizes both standard gradient descent (when $h(\vx) \equiv 0$), and projected gradient descent (when $h(\vx) = \mI_{C}(x)$ is the indicator function, equal to $0$ if $\vx \in C$, and $+\infty$ otherwise). For many penalty functions of interest, even non-differentiable and possibly non-convex ones, the proximal operator can be computed efficiently, often in closed-form. This is true in particular for the structured sparsity inducing penalties presented in this paper.

Let's take the example of the mixed $\ell_{1}/\ell_{2}$ norm. First note that since the penalty in \cref{eq:l1-l2} decomposes along the (disjoint) groups $g \in \gG$, the proximal operator in \cref{eq:proximal-operator} benefits from the same decomposition, involving only the proximal operators of the $\ell_{2}$ norm evaluated on each $\vx_{g}$. Therefore, it is sufficient to know the proximal operator of the $\ell_{2}$ norm $h(\vx) = \lambda\|\vx\|_{2}$ in order to construct the update for proximal gradient descent, which is given in closed-form \citep{combettes2005proximal} by
\begin{equation}
    \prox_{\alpha h}(\vx) = \left[1 - \frac{\alpha\lambda}{\|\vx\|_{2}}\right]_{+}\!\vx,
\end{equation}
where $[z]_{+} = \max(0, z)$. Similarly, the group MCP enjoys the same decomposition property along groups, and there exists a closed-form expression for the proximal operator of the MCP / $\ell_{2}$ penalty (i.e. $h(\vx) = \MCP(\|\vx\|_{2})$; \citealp{breheny2009groupmcp}).

\section{Adaptive proximal optimizers for structured sparsity inducing penalties}
\label{sec:adaptive-proximal-optimizers}
In practice in Deep Learning, the minimization of the objective is typically performed using stochastic and adaptive first-order methods, such as Adagrad, RMSprop, or Adam, where the update direction is rescaled with a diagonal preconditioning matrix $\mD_{t}$
\begin{equation*}
    \vx_{t+1} \leftarrow \vx_{t} - \alpha \mD_{t}^{-1}\vm_{t},
\end{equation*}
where $\vm_{t}$ is an estimate of the mean of the gradients of $f$; see also \cref{tab:optimizers} in \cref{app:adaptive-optimizers} for examples of matrices $\mD$. Although it appears to be similar to the standard gradient descent update, the preconditioning $\mD$ makes composing these methods with a proximal operator more challenging. \citet{yun2020proxgen} proposed a general proximal method that is capable of dealing with this preconditioning, with an update of the form
\begin{equation}
    \vx_{t+1} \leftarrow \prox_{\alpha h}^{\mD_{t}}\!\big(\vx_{t} - \alpha \mD_{t}^{-1}\vm_{t}\big),
\end{equation}
where $\prox_{\alpha h}^{\mD}$ is the \emph{weighted proximal operator} of $h$ \citep{hiriart1993convex,lee2014proximalnewton}
\begin{equation}
    \prox_{\alpha h}^{\mD}(\vx) = \argmin_{\vz\in\sR^{n}}\frac{1}{2}\|\vz - \vx\|_{\mD}^{2} + \alpha h(\vz),
    \label{eq:weighted-proximal-operator}
\end{equation}
and $\|\vz\|_{\mD}^{2} = \langle \vz, \mD\vz\rangle$. This naturally extends to the case where the update is stochastic, see \cref{app:algorithmic-details}, and \cref{sec:convergence-analysis} for details and additional theoretical guarantees. Contrary to the proximal operator, for which proximity is computed using the Euclidean distance, the weighted proximal operator captures the reweighting by $\mD$ by changing the geometry of the space over which proximity is measured.

However unlike their proximal operators (see \cref{sec:proximal-gradient-descent}), the weighted proximal operators of the $\ell_{2}$ norm and MCP / $\ell_{2}$ cannot be evaluated in closed-form in general. \citet{becker2019quasinewton,yang2020proxsgd} show that we can get a closed-form expression for the mixed $\ell_{1}/\ell_{2}$ norm when the preconditioning matrix $\mD$ decomposes along the groups into spherical matrices $\mD_{g} = m\mI_{g}$, which is not practical for Deep Learning applications: while being diagonal, the preconditioning matrix usually does not have the same value for all the variables belonging to the same group (e.g. it might be based on aggregated statistics from past gradients). In \cref{thm:weighted-prox-l2}, we give an expression of the weighted proximal operator for the $\ell_{2}$ norm, based on an implicit expression. Recall from \cref{sec:proximal-gradient-descent} that due to group decomposition, this is sufficient to compute the weighted proximal operator of the mixed $\ell_{1}/\ell_{2}$ norm.
\begin{theorem}[Weighted proximal operator of $\ell_{2}$]
  \label{thm:weighted-prox-l2}
  Let $\mD = \diag(d_{1}, \ldots, d_{n})$ be a positive definite diagonal matrix (i.e. $d_{i} > 0$ for all $i$), and $h(\vx) = \lambda \|\vx\|_{2}$ the $\ell_{2}$ penalty. The weighted proximal operator of $h$ is given by
  \begin{equation}
    \big[\prox_{\alpha h}^{\mD}(\vx)\big]_{i} = \left\{\begin{array}{cl}
        \dfrac{d_{i}\theta^{\star}x_{i}}{d_{i}\theta^{\star} + \alpha \lambda} & \textrm{if $\|\mD\vx\|_{2} > \alpha \lambda$}\\[1em]
        0 & \textrm{otherwise},
        \label{eq:weighted-prox-l2}
    \end{array}\right.
  \end{equation}
  where $\theta^{\star} > 0$ is the unique positive solution of
  \begin{equation}
    \sum_{i=1}^{n}\left[\frac{d_{i}x_{i}}{d_{i}\theta^{\star} + \alpha\lambda}\right]^{2} = 1.
    \label{eq:weighted-prox-l2-theta}
  \end{equation}
\end{theorem}
The proof of \cref{thm:weighted-prox-l2} is given in \cref{app:proof-weighted-prox-l2} \citep{li2020proxl2}. Unlike the (unweighted) proximal operators for many standard penalties, including for the $\ell_{2}$ norm itself, the weighted proximal operator of $\ell_{2}$ cannot be expressed in closed-form. Instead in practice, we need to solve \cref{eq:weighted-prox-l2-theta} numerically, for example using the Newton-Raphson algorithm. We note that \citet{duchi2011adagrad} also provides a bisection-based procedure to solve a somewhat different minimization problem, related to \cref{eq:weighted-proximal-operator} up to a change of variable. \cref{thm:weighted-prox-l2}, on the other hand, gives a more explicit formulation of the weighted proximal operator when the preconditioning matrix is diagonal, more appropriate for the adaptive proximal gradient method presented above \citep{yun2020proxgen}.

One consequence of this approximation of $\theta^{\star}$ is that we can only obtain an approximation of the weighted proximal operator of $\ell_{2}$, as opposed to an exact expression; the resulting proximal gradient method is then called \emph{inexact} \citep{schmidt2011inexactprox}. \cref{prop:weighted-prox-l2-approx} shows that this error on the weighted proximal operator is controlled by the approximation error induced by the root finding algorithm used to find $\theta^{\star}$.

\begin{proposition}
  Let $\varepsilon > 0$ and $\tilde{\theta} > 0$ such that $|\tilde{\theta} - \theta^{\star}| \leq \varepsilon$. Let $\widetilde{\prox}^{\mD}_{\alpha h}(\vx)$ be an approximation of the weighted proximal operator of $\ell_{2}$ (i.e. $h(\vx) = \lambda\|\vx\|_{2}$), replacing $\theta^{\star}$ by $\tilde{\theta}$ in \cref{eq:weighted-prox-l2}. Then we have for all $\vx \in \sR^{n}$:
  \begin{equation}
    \big\|\widetilde{\prox}_{\alpha h}^{\mD}(\vx) - \prox_{\alpha h}^{\mD}(\vx)\big\|_{2} \leq \varepsilon.
  \end{equation}
  \label{prop:weighted-prox-l2-approx}
  \vspace*{-2em}
\end{proposition}
The proof of \cref{prop:weighted-prox-l2-approx} is provided in \cref{app:proof-approx-weighted-prox-l2}, along with bounds on $\theta^{\star}$ to reduce the search space as much as possible, and to ensure fast convergence of the numerical solver.

Likewise, we also give in \cref{thm:weighted-prox-mcp-l2} an (implicit) expression for the weighted proximal operator of MCP / $\ell_{2}$. Again, this is sufficient to compute the weighted proximal operator of group MCP.
\begin{theorem}[Weighted proximal operator of MCP / $\ell_{2}$]
  \label{thm:weighted-prox-mcp-l2}
  Let $\mD = \diag(d_{1}, \ldots, d_{n})$ be a positive definite diagonal matrix (i.e. $d_{i} > 0$ for all $i$), and $h(\vx) = \MCP(\|\vx\|_{2})$ the MCP / $\ell_{2}$ penalty. Suppose that $\alpha$ and $\beta$ satisfy $\alpha < \beta d_{\min}$, where $d_{\min}$ is the smallest value of the diagonal of $\mD$. The weighted proximal operator of $h$ is given by
  \begin{equation}
      \scalebox{0.88}{$
      \big[\prox_{\alpha h}^{\mD}(\vx)\big]_{i} = \left\{\begin{array}{cl}
            x_{i} & \textrm{\scalebox{0.83}{if $\|\vx\|_{2} > \beta\lambda$}} \\[0.8em]
            \dfrac{d_{i}\beta\theta^{\star}x_{i}}{(d_{i}\beta - \alpha)\theta^{\star} + \alpha \beta\lambda} & \textrm{\scalebox{0.62}{\begin{tabular}{@{}l}if $\|\vx\|_{2} \leq \beta \lambda$\\[0.4em] and $\|\mD\vx\|_{2} > \alpha \lambda$\end{tabular}}}\\[1.5em]
            0 & \textrm{\scalebox{0.62}{\begin{tabular}{@{}l}if $\|\vx\|_{2} \leq \beta \lambda$\\[0.4em] and $\|\mD\vx\|_{2} \leq \alpha \lambda$,\end{tabular}}}
        \end{array}\right.$}
        \label{eq:weighted-prox-mcp-l2}
  \end{equation}
  where $\theta^{\star} > 0$ is the unique positive solution of
  \begin{equation}
      \beta^{2}\sum_{i=1}^{n}\left[\frac{d_{i}x_{i}}{(d_{i}\beta - \alpha)\theta^{\star} + \alpha\beta\lambda}\right]^{2} = 1.
      \label{eq:weighted-prox-mcp-l2-theta}
  \end{equation}
\end{theorem}
The proof is provided in \cref{app:proof-weighted-prox-mcp-l2}. Moreover, guarantees on the approximation error induced by the numerical solver to determine $\theta^{\star}$ for MCP / $\ell_{2}$, similar to \cref{prop:approx-weighted-prox-l2}, are given in \cref{app:proof-approx-weighted-prox-mcp-l2}.

\section{Convergence analysis}
\label{sec:convergence-analysis}
In this section, we will make stochasticity more explicit, and we will consider the following composite objective function
\begin{equation}
    \min_{\vx \in \sR^{N}}\big(F(\vx) \triangleq \E_{\xi}[f(\vx;\xi)] + h(\vx)\big).
\end{equation}
Prior work studied the non-asymptotic convergence properties of stochastic proximal gradient descent on this kind of objective \citep{xu2019analysisprox,yun2020proxgen}. However, their results depend on the fact that the (weighted) proximal operator of $h$ can be computed exactly, which is unfortunately impossible in practice for structured sparsity inducing penalties (their weighted version at least). We will see that these results remain valid, despite the approximation of the weighted proximal operator. Note that related to our work, convergence guarantees of (deterministic) inexact proximal gradient methods also exist, although with stronger conditions on the convexity of $f$ and $h$ \citep{schmidt2011inexactprox}.

Because the overall objective function $F$ can be non-differentiable and non-convex, convergence is proven in terms of the expected distance of the \emph{Frechet subdifferential} $\widehat{\partial}F(\vx_{t})$ to zero\footnote{This reduces to $\|\nabla F(\vx_{t})\|_{2}$ when $F$ is differentiable; see \cref{def:frechet-subdifferential} for the definition of the Frechet subdifferential. In particular at a stationary point $\vx^{\star}$, we have $\vzero \in \widehat{\partial}F(\vx^{\star})$.} \citep{rockafellar1976monotone}. To derive the convergence bound, we borrow the assumptions made by \citet{yun2020proxgen}, which are recalled here:
\begin{assumption}[\citealp{yun2020proxgen}]
    \vphantom{x}
    \begin{enumerate}
        \item \textbf{L-smoothness}\quad The loss function $f$ is $L$-smooth and lower-bounded: $f(\vx^{\star}) > -\infty$ for the optimal solution $\vx^{\star}$, and $\forall \vx, \vy, \|\nabla f(\vx) - \nabla f(\vy)\|_{2} \leq L \|\vx - \vy\|_{2}$.
        \item \textbf{Bounded variance}\quad The stochastic gradient $\vg_{t} = \nabla f(\vx_{t};\xi_{t})$ is unbiased, and has bounded variance: $\E_{\xi}[\vg_{t}] = \nabla f(\vx_{t})$, and $\E_{\xi}[\|\vg_{t} - \nabla f(\vx_{t})\|^{2}_{2}] \leq \sigma^{2}$.
        \item (i) The update and (ii) the stochastic gradient are bounded, and (iii) the momentum parameter is exponentially decaying: (i) $\|\vx_{t+1} - \vx_{t}\|_{2} \leq D$, (ii) $\|\vg_{t}\|_{2} \leq G$, (iii) $\rho_{t} = \rho_{0}\mu^{t}$, with $\mu \in [0, 1)$.
        \item \textbf{Sufficiently positive-definite}\quad For all $t$, $\mD_{t} \succeq \delta \mI$ (i.e. $\mD_{t} - \delta \mI$ is positive semi-definite), and $\alpha_{t}\mD_{t}^{-1} \succeq \gamma \mI$, with $\delta, \gamma > 0$.
    \end{enumerate}
    \label{hyp:proxgen-assumptions}
\end{assumption}
In addition to these assumptions, we also need to account for the possible approximation of the weighted proximal operator of the structured sparsity inducing penalties. Here, we add the following assumption
\begin{assumption}
    Let $\vx_{t+1}^{\star} = \prox_{\alpha h}^{\mD_{t}}\!\big(\vx_{t} - \alpha \mD_{t}^{-1}\vm_{t}\big)$ be the exact proximal update. For all $t \geq 0$, we have either
    \begin{enumerate}
        \item $\vx_{t+1}$ is the exact proximal update: $\vx_{t+1} \leftarrow \vx_{t+1}^{\star}$;
        \item or $\vx_{t+1}$ is an $\varepsilon_{t+1}$-approximation of $\vx_{t+1}^{\star}$, $\|\vx_{t+1} - \vx_{t+1}^{\star}\|_{2} \leq \varepsilon_{t+1}$, and $h$ is $L'$-smooth in the $\varepsilon_{t+1}$-ball $\gB$ around $\vx_{t+1}^{\star}$ (i.e. $\vx_{t+1} \in \gB$).
    \end{enumerate}
    Moreover, we assume that the approximation of the weighted proximal operator becomes more accurate as $t$ grows. Specifically, with the convention $\varepsilon_{t+1} = 0$ if $\vx_{t+1} = \vx_{t+1}^{\star}$,
    \begin{equation*}
        \sum_{t=0}^{+\infty} \varepsilon_{t+1}^{2} = K < +\infty.
    \end{equation*}
    \label{hyp:prox-approximation}
    \vspace*{-1em}
\end{assumption}
\cref{hyp:prox-approximation} is a slightly weaker assumption than that of \citet{yun2020proxgen}: instead of always relying on an exact proximal update ($\vx_{t+1} = \vx_{t+1}^{\star}$), we allow the update to be approximated at points where the penalty is smooth. This condition is in particular satisfied for both structured sparsity inducing penalties studied in this paper, since the value of $\theta^{\star}$, which is the only approximated quantity, is only used to define the weighted proximal operator in a regime where the corresponding penalty is smooth (at least if $\|\mD\vx\|_{2} > \alpha \lambda$). The decreasing sequence of approximation errors is also a well accepted assumption in the context of inexact proximal methods \citep{rockafellar1976monotone}. Under these conditions, we can derive a bound on the convergence of stochastic proximal methods, similar to the one given in \citet{xu2019analysisprox}. Note that \cref{alg:minimization-composite-loss} is given in \cref{app:algorithmic-details}.
\begin{theorem}
    Suppose that \cref{hyp:proxgen-assumptions,hyp:prox-approximation} are satisfied. If we run \cref{alg:minimization-composite-loss} with a non-increasing step-size $\alpha_{t}$, such that $\alpha_{0} < \delta/2L$, then the output $\vx_{R}$ of \cref{alg:minimization-composite-loss}, where $R$ is sampled uniformly in $\{1, \ldots, T\}$, satisfies
    \begin{equation*}
        \scalebox{0.85}{$\displaystyle \E_{R}\big[\dist\!\big(\vzero, \widehat{\partial}F(\vx_{R})\big)^{2}\big] \leq \frac{C_{1}}{T}\sum_{t=0}^{T-1}\|\vg_{t} - \nabla f(\vx_{t})\|_{2}^{2} + \frac{C_{2}\Delta}{T} + \frac{C_{3}}{T},$}
    \end{equation*}
    where $\Delta = F(\vx_{0}) - F(\vx^{\star})$ ($\vx^{\star}$ is a solution of \cref{eq:composite-loss}), and with $C_{1}$, $C_{2}$, and $C_{3}$ positive constants independent of $T$. Here $\dist(\vz, S)$ is the distance of a set $S$ to a point $\vz$, defined as the minimal distance of any point in $S$ to $\vz$.
    \label{thm:convergence-analysis}
\end{theorem}
The proof of \cref{thm:convergence-analysis} is given in \cref{app:proof-convergence-analysis}. We note that this matches previous (non-asymptotic) convergence guarantees of adaptive proximal gradient methods from \citet{yun2020proxgen}, albeit with the slightly weaker \cref{hyp:prox-approximation}. Moreover, the corollary results from \citet{xu2019analysisprox} of the convergence of mini-batch stochastic proximal methods can be directly transposed here; in particular, in the case where the mini-batch size is fixed:
\begin{corollary}[Fixed mini-batch size]
    If the assumptions of \cref{thm:convergence-analysis} are satisfied, with $T = 2(C_{2}\Delta + C_{3})/\varepsilon^{2}$ and with a fixed mini-batch size $m_{t}$ with $m_{t} = 2C_{1}\sigma^{2}/\varepsilon^{2}$, then the output $\vx_{R}$ of \cref{alg:minimization-composite-loss} satisfies
    \begin{equation*}
        \E_{R}\big[\dist\!\big(\vzero, \widehat{\partial}F(\vx_{R})\big)^{2}\big] \leq \varepsilon^{2},
    \end{equation*}
    where $C_{1}$, $C_{2}$, and $C_{3}$ are the constants from \cref{thm:convergence-analysis}. To have $\E[\dist(\vzero, \widehat{\partial}F(\vx_{R}))] \leq \varepsilon$, it is then sufficient to have $T = O(1 / \varepsilon^{2})$, making the total complexity $O(1 / \varepsilon^{4})$.
    \label{cor:fixed-batch-size}
\end{corollary}

\section{Experimental results}
\label{sec:experimental-results}
To validate that proximal gradient methods with structured sparsity inducing penalties are indeed capable of finding solutions with structure in their sparsity patterns, we experimented with two families of representative architectures: convolutional neural networks from computer vision, and transformers from natural language processing. In all experiments with proximal gradient methods, we used the \textsc{ProxGen} algorithm \citep{yun2020proxgen} presented in \cref{sec:adaptive-proximal-optimizers}, and the weighted proximal operators given in \cref{thm:weighted-prox-l2,thm:weighted-prox-mcp-l2}; the adaptive optimizer (i.e. the form of the preconditioning matrix) is Adam \citep{kingma2014adam}. In order to approximate the weighted proximal operators, we used the Newton-Raphson algorithm for its fast convergence properties; see \cref{alg:newton-raphson-l2} for details about this procedure. We will get back to the choice of this algorithm in \cref{sec:approximation-weighted-prox}.

\subsection{Convolutional Neural Networks}
\label{sec:convolutional-neural-networks}
To show the advantage of structured sparsity in convolutional architectures, we trained a VGG-16 \citep{simonyan2015vgg} on CIFAR-10 using both the mixed $\ell_{1}/\ell_{2}$ norm and group MCP. The penalty is only applied to the weights of the neural network, leaving the biases unpenalized. Following \citet{wen2016structured}, we study \emph{channel-wise} structured sparsity for the convolutional layers, where the groups correspond to all the outgoing weights of a single channel; see \cref{fig:structured-sparsity} for an illustration. Likewise, we use a \emph{row-wise} structure for the weights of the unique fully-connected layer. This structure is motivated by finding a network where intermediate representations are only influenced by a subset of channels from the previous layer; we will return to the advantages of this choice in \cref{sec:pruning-indirect-sparsity}. Overall, there are 4k groups of variables, in a model containing about 15M parameters. We also experimented with Residual Networks, with similar sizes and groups; details and experimental results are available in \cref{app:residual-networks}.

\begin{table}[t]
    \vspace*{-1em}
    \centering
    \caption{Performance of VGG-16 trained on CIFAR-10, with different structured sparsity inducing penalties, reported as the mean and standard deviation over 3 runs. For each penalty, the first line corresponds to training with subgradient methods \citep{wen2016structured}, and ``+ prox.'' with proximal gradient methods. Here, group sparsity is the proportion of groups (out of 4k) with non-zero norm. ${}^{\star}$Required a thresholding step, see the text for details.}
    \vspace*{0.1in}
    \begin{tabular}{lcc}
        \toprule
         & Group Sparsity & Test accuracy\\
        \midrule
        Baseline & -- & $90.76 \pm 0.29 \%\phantom{{}^{\star}}$ \\
        \midrule
        $\ell_{1}/\ell_{2}$ & $20.74 \pm 0.30 \%^{\star}$ & $89.53 \pm 0.43 \%^{\star}$  \\
        \hspace{2em}+ prox. & $22.03 \pm 0.33 \%\phantom{{}^{\star}}$ & $89.55 \pm 0.14 \%\phantom{{}^{\star}}$ \\[0.3em]
        Group MCP & $20.93 \pm 0.30 \%^{\star}$ & $89.59 \pm 0.30\%^{\star}$ \\
        \hspace{2em}+ prox. & $22.63 \pm 0.15 \%\phantom{{}^{\star}}$ & $89.80 \pm 0.06 \%\phantom{{}^{\star}}$ \\
        \bottomrule
    \end{tabular}
    \label{tab:vgg16-group-sparsity}
\end{table}

With this experiment, we also want to show the effectiveness of proximal gradient methods with the (non-differentiable) penalties studied in this paper, as opposed to subgradient methods \citep{wen2016structured}, where the composite objective was simply trained with Adam. \cref{tab:vgg16-group-sparsity} compares the performance for both training methods, with both structured sparsity inducing penalties. In the case of subgradient methods, the solution found for both penalties was not sparse, i.e. no group had zero norm, meaning a group sparsity of $100\%$. Following \citet{wen2016structured}, we applied a post-processing thresholding step to these networks, where groups with small enough $\ell_{2}$ norm were zeroed-out. We found that this step was highly sensitive to the choice of the threshold, jumping from $90\%$ accuracy down to $10\%$ (i.e. a random predictor) over a small range of values; the values reported in \cref{tab:vgg16-group-sparsity} are trading-off accuracy for group sparsity.

In contrast to subgradient methods, the solutions found using proximal gradient methods are capable of reaching a low level of group sparsity, without any post-processing necessary (because the proximal operator can set groups of parameters to $0$), while maintaining a good accuracy compared to our baseline model.

\begin{figure*}[t]
    \centering
    \includegraphics[width=0.463\linewidth]{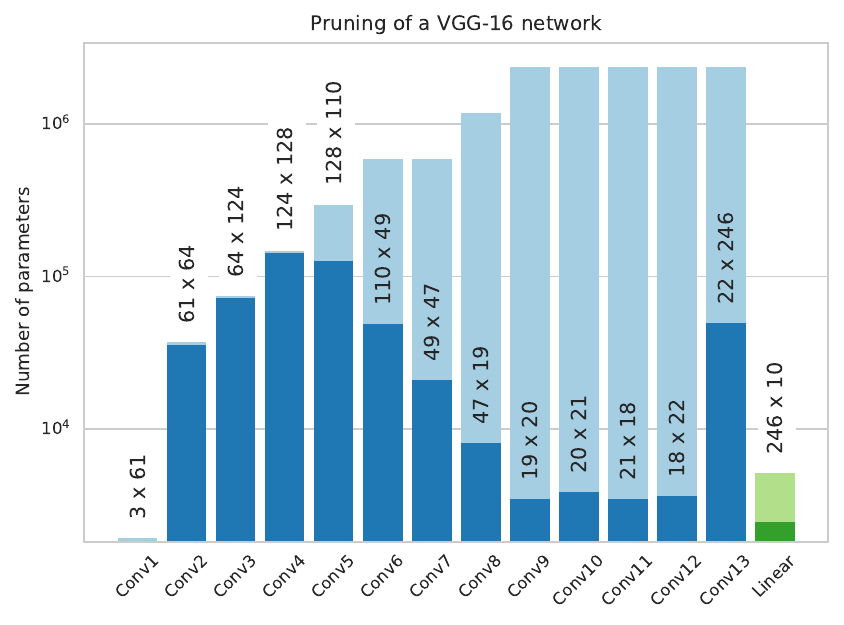}
    \includegraphics[width=0.45\linewidth]{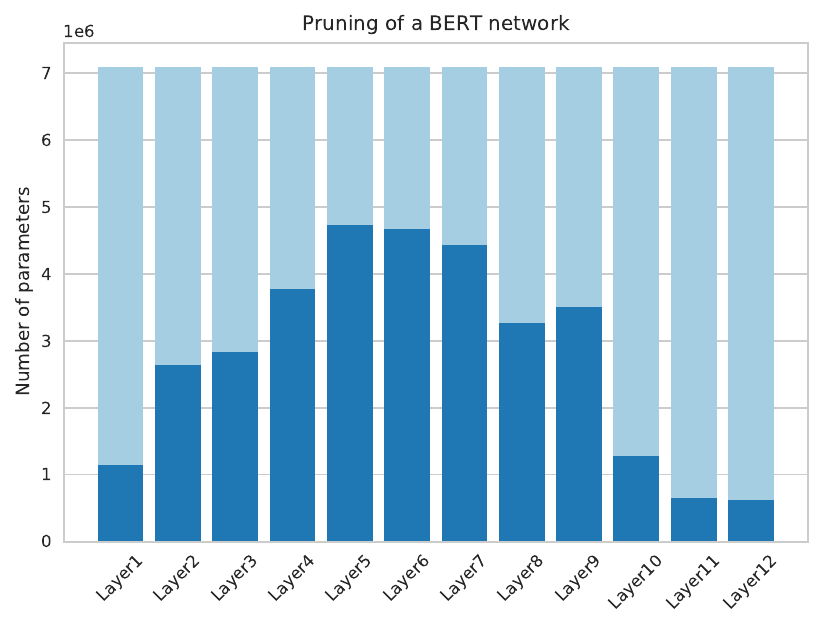}
    \vspace{-1em}
    \caption{Comparison of the size of the network before (light) and after (dark) pruning, for each layer of VGG-16 (left) and BERT (right). For VGG-16, convolutional layers (with their corresponding batch-normalization) are represented in blue, and the linear layer in green; note that since there is only a single linear layer at the end, this network has only 14 layers. The label above each bar represents the size of the layer after pruning (all convolutions have a kernel of size $3\times 3$).}
    \label{fig:pruned-vgg16-bert}
\end{figure*}

\subsection{Large-scale Transformers}
\label{sec:transformers}
Unlike in the large-scale computer vision community, which has mostly moved away from adaptive optimizers \citep{wilson2017marginalvalue}, Adam is still a popular choice for training and fine-tuning large-scale transformers for language modeling. To test our inexact proximal gradient method on a large-scale problem, we fine-tuned a BERT model \citep{devlin2018bert} on SQuAD 1.1, a question-answering benchmark \citep{rajpurkar2016squad}; the language model was initially pre-trained on a large corpus of articles from Wikipedia and the BookCorpus dataset. Taking inspiration from our experiment with VGG-16, we use a row-wise structure for the weight matrices of all the attention layers in the network. In the context of BERT, \citet{guo2019reweightedl1} found that this kind of structure was already emerging in some weight matrices with unstructured sparsity inducing penalties. Overall, there are 86k groups of variables, in a model containing 108M parameters. To be consistent with prior work on pruning language models \citep{sanh2020movementpruning}, the embeddings are kept fixed during fine-tuning.

\begin{table}[ht]
    \vspace*{-1em}
    \centering
    \caption{Performance of BERT fine-tuned on SQuAD, with different structured sparsity inducing penalties, and different levels of group sparsity. Here, group sparsity is the proportion of groups (out of 86k) with non-zero norm.}
    \vspace*{0.1in}
    \begin{tabular}{lccc}
        \toprule
         & Group Sparsity & EM & F1\\
        \midrule
        Baseline & -- & $81.01$ & $88.27$ \\
        \midrule
        \multirow{2}{*}{$\ell_{1}/\ell_{2}$} & $84.34\%$ & $72.29$ & $81.96$ \\
        & $76.14\%$ & $66.16$ & $77.50$ \\[0.2em]
        \multirow{2}{*}{Group MCP} & $81.49\%$ & $75.40$ & $84.50$ \\
        & $61.33\%$ & $69.47$ & $79.95$ \\
        \bottomrule
    \end{tabular}
    \label{tab:squad-group-sparsity}
\end{table}

\cref{tab:squad-group-sparsity} shows the performance of BERT, both in terms of F1 score and Exact Match (EM), for two choices of hyperparameters ($\lambda$ and $\beta$), leading to two values of group sparsity. In contrast to our experiments in \cref{sec:convolutional-neural-networks}, the levels of group sparsity reached using the structured sparsity inducing penalties are higher, meaning that fewer parameters are zeroed-out. The gap in performance between the sparse models and the baseline matches the gap found in prior work using unstructured sparsity \citep{sanh2020movementpruning}, although with more limited levels of sparsity here. We can also observe that at comparable levels of group sparsity, the solutions found with group MCP tend to perform better than the ones found with the mixed $\ell_{1}/\ell_{2}$ norm. These results are encouraging, and show that structured sparsity can also be effective on large-scale models, and remains practical even with a very large number of groups.

\subsection{Pruning with indirect sparsity}
\label{sec:pruning-indirect-sparsity}
Interestingly, as observed by \citet{li2017pruning}, enforcing channel-wise and row-wise structured sparsity at the level of one layer has an effect on the computational efficiency of that one layer, but it has consequences on the neighboring layers as well: if one channel sees its outgoing weights being zeroed-out, then it does not contribute to the output of the network anymore, and can therefore be ignored in upstream computations as well. This has a significant impact on the effective size of the network, beyond group sparsity, where connections can be further pruned \emph{indirectly} thanks to these structured sparsity inducing penalties without affecting predictions. The details of this pruning procedure are given in \cref{app:pruning-indirect-sparsity}.

\cref{fig:pruned-vgg16-bert} shows the proportion of parameters preserved after applying this procedure on two networks: VGG-19 trained with $\ell_{1}/\ell_{2}$ (group sparsity $22.03\%)$, and BERT fine-tuned with group MCP (group sparsity $61.33\%$). The effective sparsity (i.e. the size of the pruned network, divided by the size of the original network) is about $3\%$ for VGG-16, and $48\%$ for BERT, with significant gains on some layers (e.g. 3 orders of magnitude fewer parameters for Conv9-12 in VGG-16). For BERT, we are still below the high (unstructured) sparsity of \citet{sanh2020movementpruning}, where the same level of performance is reached with about $3\%$ of the size of the original network. To achieve better sparsity with structured sparsity inducing penalties, Transformers in general could benefit from better choices of groups, e.g. spanning over multiple weight matrices; see \cref{app:large-scale-transformers} for further discussion.

\begin{figure}[t]
    \centering
    \includegraphics[width=0.95\linewidth]{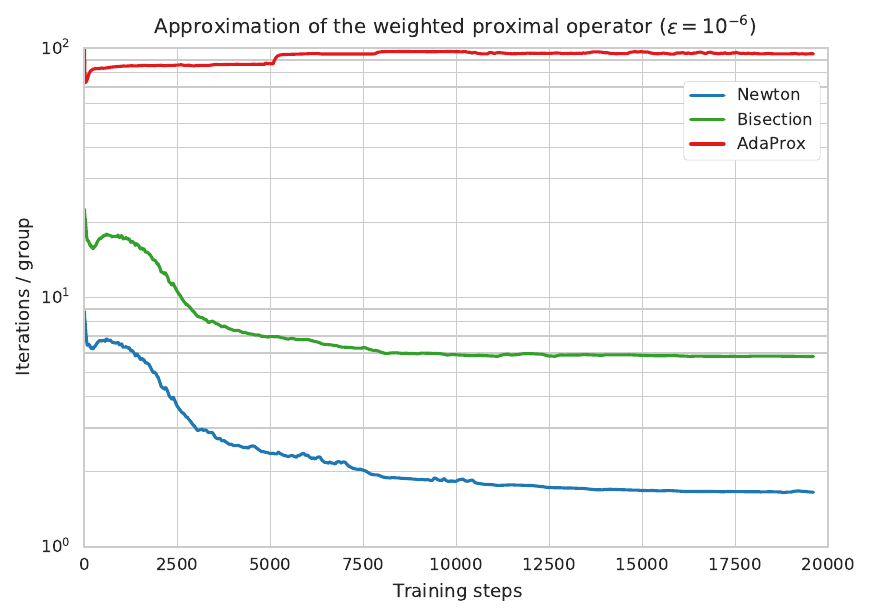}
    \vspace{-1em}
    \caption{Number of iterations, normalized by the total number of groups (here 4k), to approximate the weighted proximal operator with 3 different algorithms, while training a VGG-16 network on CIFAR-10. Note that Newton's algorithm reaches a number of iterations per group close to 1 at the end of training (about 1.65 iterations / group), in part thanks to the amortized cost of running the algorithm in a vectorized way.}
    \label{fig:approximation-weighted-prox}
    \vspace*{-1em}
\end{figure}

However, a big advantage of structured sparsity, with our choice of groups, is that the pruned models are smaller \emph{dense} networks: in addition to having a smaller memory footprint, this can accelerate inference, without any special requirement in terms of hardware (see \citet{sanh2020movementpruning} for discussions).

\subsection{Approximation of the weighted proximal operator}
\label{sec:approximation-weighted-prox}
Although our formulation of the weighted proximal operators for the mixed $\ell_{1}/\ell_{2}$ norm and group MCP in \cref{thm:weighted-prox-l2,thm:weighted-prox-mcp-l2} is independent of the choice of the numerical solver to find $\theta^{\star}$, in practice we chose the Newton-Raphson algorithm. In this section, we evaluate empirically how efficient this algorithm is over the course of training, compared to other approximations of the weighted proximal operator. This also helps us to estimate the overhead induced by this routine over standard first-order methods. 

As a direct replacement of Newton's algorithm, we consider a bisection-based method, similar to the one described in \citet{duchi2011adagrad}. We also compare it to AdaProx \citep{melchior2019adaprox}, a general method that approximates the weighted proximal operator by iterating over the (unweighted) proximal operator. Since all three methods are iterative methods, we evaluate their efficiency as the number of iterations per gradient update. Similar to \cref{sec:convolutional-neural-networks}, we trained a VGG-16 network on CIFAR-10 with the mixed $\ell_{1}/\ell_{2}$ norm as our penalty. In order to control for variations during the learning process, at each gradient update we approximate the weighted proximal operator with all three methods, but we only apply the one found by Newton's algorithm. We also use the same tolerance $\varepsilon = 10^{-6}$ for all algorithms.

\cref{fig:approximation-weighted-prox} shows the evolution of the number of iterations per group, and per gradient update, with the number of training steps. Thanks to the increasing group sparsity, the cost induced by both Newton's algorithm and bisection decreases over time: in the case of $\ell_{1}/\ell_{2}$, only non-zero $\prox(\vx)$ require an approximation of $\theta^{\star}$ (i.e. if $\|\mD\vx\|_{2} > \alpha\lambda$; see \cref{thm:weighted-prox-l2}). AdaProx, on the other hand, iterates over the proximal operator of $\ell_{2}$ regardless of our knowledge of $\|\mD\vx\|_{2}$, and therefore has a non-decreasing cost over the course of training. This cost is also significantly higher than the other two methods, because AdaProx was often reaching the maximum number of iterations, set at $100$ per gradient update. Finally, Newton's algorithm is about 3 times as fast as the bisection-based method, showing a clear advantage in terms of efficiency, and supporting our choice of this algorithm as part of our inexact proximal gradient method.

\section{Conclusion}
\label{sec:conclusion}
The weighted proximal operator is a quantity that plays a fundamental role in adaptive proximal gradient methods, a class of methods derived from adaptive optimizers in Deep Learning, specifically adapted for optimizing composite objectives with non-differentiable penalties. In this work, we derived the weighted proximal operators of two structured sparsity inducing penalties: the mixed $\ell_{1}/\ell_{2}$ norm, and group MCP. We saw that unlike their unweighted counterparts, there exists no closed-form expression for the weighted proximal operators of these penalties. Nevertheless, we found that they can be efficiently approximated using the Newton-Raphson algorithm, and they can be incorporated as part of a general proximal gradient method called ProxGen \citep{yun2020proxgen}. We also proved that the convergence guarantees of this algorithm were maintained, despite the resulting method being inexact. We showed experimentally that this proximal gradient method applied to structured sparsity inducing penalties was capable of finding sparse solutions, with structure in their sparsity patterns, on examples from computer vision with VGG-16, and natural language processing with BERT, with a moderate loss of performance. Finally, we leveraged this structure to prune these models even further, and obtain much smaller models while keeping them functionally identical.

Going beyond pruning neural networks, proximal gradient methods, together with the weighted proximal operators derived here, can be applied to any problem where structure is known or expected a priori. In future work, we would like to evaluate this method on other choices of groups, for example motivated by the causal structure of the problem.

\bibliography{references}

\begin{thebibliography}{45}
\providecommand{\natexlab}[1]{#1}
\providecommand{\url}[1]{\texttt{#1}}
\expandafter\ifx\csname urlstyle\endcsname\relax
  \providecommand{\doi}[1]{doi: #1}\else
  \providecommand{\doi}{doi: \begingroup \urlstyle{rm}\Url}\fi

\bibitem[Bach et~al.(2012{\natexlab{a}})Bach, Jenatton, Mairal, and
  Obozinski]{bach2011optimization}
Bach, F., Jenatton, R., Mairal, J., and Obozinski, G.
\newblock {Optimization with Sparsity-Inducing Penalties}.
\newblock \emph{Foundations and Trends in Machine Learning},
  2012{\natexlab{a}}.

\bibitem[Bach et~al.(2012{\natexlab{b}})Bach, Jenatton, Mairal, Obozinski,
  et~al.]{bach2012structuredconvexopt}
Bach, F., Jenatton, R., Mairal, J., Obozinski, G., et~al.
\newblock {Structured Sparsity through Convex Optimization}.
\newblock \emph{Statistical Science}, 2012{\natexlab{b}}.

\bibitem[Becker et~al.(2019)Becker, Fadili, and Ochs]{becker2019quasinewton}
Becker, S., Fadili, J., and Ochs, P.
\newblock {On Quasi-Newton Forward-Backward Splitting: Proximal Calculus and
  Convergence}.
\newblock \emph{SIAM Journal on Optimization}, 2019.

\bibitem[Breheny \& Huang(2009)Breheny and Huang]{breheny2009groupmcp}
Breheny, P. and Huang, J.
\newblock {Penalized methods for bi-level variable selection}.
\newblock \emph{Statistics and its Interface}, 2009.

\bibitem[Chen et~al.(2020)Chen, Wang, Ding, Ji, Yi, and Zhu]{chen2020half}
Chen, T., Wang, G., Ding, T., Ji, B., Yi, S., and Zhu, Z.
\newblock Half-space proximal stochastic gradient method for group-sparsity
  regularized problem.
\newblock \emph{arXiv preprint}, 2020.

\bibitem[Combettes \& Wajs(2005)Combettes and Wajs]{combettes2005proximal}
Combettes, P.~L. and Wajs, V.~R.
\newblock {Signal recovery by proximal forward-backward splitting}.
\newblock \emph{Multiscale Modeling \& Simulation}, 2005.

\bibitem[Copas(1983)]{copas1983shrinkage}
Copas, J.~B.
\newblock {Regression, Prediction and Shrinkage}.
\newblock \emph{Journal of the Royal Statistical Society: Series B
  (Methodological)}, 1983.

\bibitem[Devlin et~al.(2018)Devlin, Chang, Lee, and Toutanova]{devlin2018bert}
Devlin, J., Chang, M.-W., Lee, K., and Toutanova, K.
\newblock {BERT: Pre-training of Deep Bidirectional Transformers for Language
  Understanding}.
\newblock \emph{arXiv preprint}, 2018.

\bibitem[Duchi et~al.(2011)Duchi, Hazan, and Singer]{duchi2011adagrad}
Duchi, J., Hazan, E., and Singer, Y.
\newblock {Adaptive Subgradient Methods for Online Learning and Stochastic
  Optimization}.
\newblock \emph{Journal of Machine Learning Research}, 2011.

\bibitem[Duchi \& Singer(2009)Duchi and Singer]{singer2009fobos}
Duchi, J.~C. and Singer, Y.
\newblock {Efficient Learning using Forward-Backward Splitting}.
\newblock In \emph{Advances in Neural Information Processing Systems}, 2009.

\bibitem[Germain et~al.(2015)Germain, Gregor, Murray, and
  Larochelle]{germain2015made}
Germain, M., Gregor, K., Murray, I., and Larochelle, H.
\newblock {MADE: Masked Autoencoder for Distribution Estimation}.
\newblock In \emph{International Conference on Machine Learning}, 2015.

\bibitem[Guo et~al.(2019)Guo, Liu, Mungall, Lin, and Wang]{guo2019reweightedl1}
Guo, F.-M., Liu, S., Mungall, F.~S., Lin, X., and Wang, Y.
\newblock {Reweighted Proximal Pruning for Large-Scale Language
  Representation}.
\newblock \emph{arXiv preprint}, 2019.

\bibitem[Hiriart-Urruty \& Lemar{\'e}chal(1993)Hiriart-Urruty and
  Lemar{\'e}chal]{hiriart1993convex}
Hiriart-Urruty, J.-B. and Lemar{\'e}chal, C.
\newblock \emph{{Convex Analysis and Minimization Algorithms II: Advanced
  Theory and Bundle Methods}}.
\newblock Springer, 1993.

\bibitem[Ioannou et~al.(2017)Ioannou, Robertson, Cipolla, and
  Criminisi]{ioannou2017deeproots}
Ioannou, Y., Robertson, D., Cipolla, R., and Criminisi, A.
\newblock {Deep Roots: Improving CNN Efficiency with Hierarchical Filter
  Groups}.
\newblock In \emph{Proceedings of the IEEE Conference on Computer Vision and
  Pattern Recognition}, 2017.

\bibitem[Ke et~al.(2019)Ke, Bilaniuk, Goyal, Bauer, Larochelle, Sch{\"o}lkopf,
  Mozer, Pal, and Bengio]{ke2019neuralcausal}
Ke, N.~R., Bilaniuk, O., Goyal, A., Bauer, S., Larochelle, H., Sch{\"o}lkopf,
  B., Mozer, M.~C., Pal, C., and Bengio, Y.
\newblock {Learning Neural Causal Models from Unknown Interventions}.
\newblock \emph{arXiv preprint}, 2019.

\bibitem[Kingma \& Ba(2015)Kingma and Ba]{kingma2014adam}
Kingma, D.~P. and Ba, J.
\newblock {Adam: A Method for Stochastic Optimization}.
\newblock \emph{International Conference on Learning Representations}, 2015.

\bibitem[Krizhevsky et~al.(2012)Krizhevsky, Sutskever, and
  Hinton]{krizhevsky2012alexnet}
Krizhevsky, A., Sutskever, I., and Hinton, G.~E.
\newblock {ImageNet Classification with Deep Convolutional Neural Networks}.
\newblock In \emph{Advances in Neural Information Processing Systems}, 2012.

\bibitem[Kyrillidis et~al.(2015)Kyrillidis, Baldassarre, El~Halabi, Tran-Dinh,
  and Cevher]{kyrillidis2015structuredsparsity}
Kyrillidis, A., Baldassarre, L., El~Halabi, M., Tran-Dinh, Q., and Cevher, V.
\newblock {Structured Sparsity: Discrete and Convex approaches}.
\newblock In \emph{Compressed Sensing and its Applications}. Springer, 2015.

\bibitem[Lachapelle et~al.(2020)Lachapelle, Brouillard, Deleu, and
  Lacoste-Julien]{lachapelle2020grandag}
Lachapelle, S., Brouillard, P., Deleu, T., and Lacoste-Julien, S.
\newblock {Gradient-based Neural DAG Learning}.
\newblock \emph{International Conference on Learning Representations}, 2020.

\bibitem[Lee et~al.(2014)Lee, Sun, and Saunders]{lee2014proximalnewton}
Lee, J.~D., Sun, Y., and Saunders, M.~A.
\newblock {Proximal Newton-type methods for minimizing composite functions}.
\newblock \emph{SIAM Journal on Optimization}, 2014.

\bibitem[Lee \& Lee(2019)Lee and Lee]{lee2019proxadam}
Lee, S. and Lee, J.
\newblock {Compressed Learning of Deep Neural Networks for OpenCL-Capable
  Embedded Systems}.
\newblock \emph{Applied Sciences}, 2019.

\bibitem[Li et~al.(2017)Li, Kadav, Durdanovic, Samet, and Graf]{li2017pruning}
Li, H., Kadav, A., Durdanovic, I., Samet, H., and Graf, H.~P.
\newblock {Pruning Filters for Efficient Convnets}.
\newblock \emph{International Conference on Learning Representations}, 2017.

\bibitem[Li(2020)]{li2020proxl2}
Li, R.
\newblock Proximal operator of $ f \left( x \right) = {\left\| a x
  \right\|}_{2} $ where $ a $ is diagonal matrix (weighted $ {L}_{2} $ norm).
\newblock Mathematics Stack Exchange, 2020.
\newblock URL \url{https://math.stackexchange.com/q/3582685}.
\newblock (version: 2020-03-16).

\bibitem[Lin et~al.(2019)Lin, Ji, Li, Deng, and Li]{lin2019toward}
Lin, S., Ji, R., Li, Y., Deng, C., and Li, X.
\newblock {Toward Compact ConvNets via Structure-Sparsity Regularized Filter
  Pruning}.
\newblock \emph{IEEE Transactions on Neural Networks and Learning Systems},
  2019.

\bibitem[Melchior et~al.(2019)Melchior, Joseph, and
  Moolekamp]{melchior2019adaprox}
Melchior, P., Joseph, R., and Moolekamp, F.
\newblock {Proximal Adam: Robust Adaptive Update Scheme for Constrained
  Optimization}.
\newblock \emph{arXiv preprint}, 2019.

\bibitem[Moreau(1962)]{moreau1962proximal}
Moreau, J.~J.
\newblock {Fonctions convexes duales et points proximaux dans un espace
  hilbertien}.
\newblock \emph{Elsevier}, 1962.

\bibitem[Murphy(2012)]{murphy2012mlapp}
Murphy, K.~P.
\newblock \emph{{Machine Learning: A Probabilistic Perspective}}.
\newblock MIT press, 2012.

\bibitem[Parikh \& Boyd(2014)Parikh and Boyd]{parikh2014proximal}
Parikh, N. and Boyd, S.
\newblock {Proximal Algorithms}.
\newblock \emph{Foundations and Trends in Optimization}, 2014.

\bibitem[Rajpurkar et~al.(2016)Rajpurkar, Zhang, Lopyrev, and
  Liang]{rajpurkar2016squad}
Rajpurkar, P., Zhang, J., Lopyrev, K., and Liang, P.
\newblock {SQuAD: 100,000+ Questions for Machine Comprehension of Text}.
\newblock \emph{arXiv preprint}, 2016.

\bibitem[Rockafellar(1976)]{rockafellar1976monotone}
Rockafellar, R.~T.
\newblock Monotone operators and the proximal point algorithm.
\newblock \emph{SIAM journal on control and optimization}, 1976.

\bibitem[Rockafellar \& Wets(2009)Rockafellar and
  Wets]{rockafellar2009variational}
Rockafellar, R.~T. and Wets, R. J.-B.
\newblock \emph{{Variational Analysis}}.
\newblock {Springer Science \& Business Media}, 2009.

\bibitem[Rumelhart et~al.(1986)Rumelhart, Hinton, and
  Williams]{rumelhart1986momentum}
Rumelhart, D.~E., Hinton, G.~E., and Williams, R.~J.
\newblock {Learning representations by back-propagating errors}.
\newblock \emph{Nature}, 1986.

\bibitem[Sanh et~al.(2020)Sanh, Wolf, and Rush]{sanh2020movementpruning}
Sanh, V., Wolf, T., and Rush, A.~M.
\newblock {Movement Pruning: Adaptive Sparsity by Fine-Tuning}.
\newblock \emph{Advances in Neural Information Processing Systems}, 2020.

\bibitem[Schmidt et~al.(2011)Schmidt, Roux, and Bach]{schmidt2011inexactprox}
Schmidt, M., Roux, N.~L., and Bach, F.
\newblock {Convergence rates of inexact proximal-gradient methods for convex
  optimization}.
\newblock \emph{arXiv preprint}, 2011.

\bibitem[Simonyan \& Zisserman(2015)Simonyan and Zisserman]{simonyan2015vgg}
Simonyan, K. and Zisserman, A.
\newblock {Very Deep Convolutional Networks for Large-Scale Image Recognition}.
\newblock \emph{International Conference on Learning Representations}, 2015.

\bibitem[Tieleman \& Hinton(2012)Tieleman and Hinton]{tieleman2012rmsprop}
Tieleman, T. and Hinton, G.
\newblock {RMSprop: Divide the gradient by a running average of its recent
  magnitude}.
\newblock \emph{Neural networks for machine learning}, 2012.

\bibitem[Wen et~al.(2016)Wen, Wu, Wang, Chen, and Li]{wen2016structured}
Wen, W., Wu, C., Wang, Y., Chen, Y., and Li, H.
\newblock {Learning Structured Sparsity in Deep Neural Networks}.
\newblock In \emph{Advances in Neural Information Processing Systems}, 2016.

\bibitem[Wilson et~al.(2017)Wilson, Roelofs, Stern, Srebro, and
  Recht]{wilson2017marginalvalue}
Wilson, A.~C., Roelofs, R., Stern, M., Srebro, N., and Recht, B.
\newblock {The Marginal Value of Adaptive Gradient Methods in Machine
  Learning}.
\newblock In \emph{Advances in Neural Information Processing Systems}, 2017.

\bibitem[Xu et~al.(2019)Xu, Jin, and Yang]{xu2019analysisprox}
Xu, Y., Jin, R., and Yang, T.
\newblock {Non-asymptotic Analysis of Stochastic Methods for Non-Smooth
  Non-Convex Regularized Problems}.
\newblock In \emph{Advances in Neural Information Processing Systems}, 2019.

\bibitem[Yang et~al.(2020)Yang, Yuan, Chatzimichailidis, van Sloun, Lei, and
  Chatzinotas]{yang2020proxsgd}
Yang, Y., Yuan, Y., Chatzimichailidis, A., van Sloun, R.~J., Lei, L., and
  Chatzinotas, S.
\newblock {ProxSGD: Training Structured Neural Networks under Regularization
  and Constraints}.
\newblock In \emph{International Conference on Learning Representations}, 2020.

\bibitem[Yuan \& Lin(2006)Yuan and Lin]{yuan2006grouplasso}
Yuan, M. and Lin, Y.
\newblock Model selection and estimation in regression with grouped variables.
\newblock \emph{Journal of the Royal Statistical Society: Series B (Statistical
  Methodology)}, 2006.

\bibitem[Yun et~al.(2020)Yun, Lozano, and Yang]{yun2020proxgen}
Yun, J., Lozano, A.~C., and Yang, E.
\newblock {A General Family of Stochastic Proximal Gradient Methods for Deep
  Learning}.
\newblock \emph{arXiv preprint}, 2020.

\bibitem[Zhang(2010)]{zhang2010mcp}
Zhang, C.-H.
\newblock {Nearly Unbiased Variable Selection under Minimax Concave Penalty}.
\newblock \emph{The Annals of Statistics}, 2010.

\bibitem[Zhou et~al.(2016)Zhou, Alvarez, and Porikli]{zhou2016lessismore}
Zhou, H., Alvarez, J.~M., and Porikli, F.
\newblock {Less is More: Towards Compact CNNs}.
\newblock \emph{European Conference on Computer Vision}, 2016.

\bibitem[Zou \& Hastie(2005)Zou and Hastie]{zou2005elasticnet}
Zou, H. and Hastie, T.
\newblock {Regularization and Variable Selection via the Elastic Net}.
\newblock \emph{Journal of the Royal Statistical Society: Series B (Statistical
  Methodology)}, 2005.

\end{thebibliography}
\bibliographystyle{icml2021}

\newpage

\appendix
\onecolumn

\section{Proofs}
\label{app:proofs}
In this section, we prove the main results from the paper, namely \cref{thm:weighted-prox-l2,thm:weighted-prox-mcp-l2}, and \cref{prop:weighted-prox-l2-approx}. In addition to these results, we also prove additional results on the approximation error for MCP / $\ell_{2}$ (i.e. the equivalent of \cref{prop:weighted-prox-l2-approx} for MCP / $\ell_{2}$), as well as bounds on the different $\theta^{\star}$ used in the main theorems, which can be used to narrow down the search space for the root-finding algorithm. The proofs of these two theorems are based on the following lemma, which gives an expression of the weighted proximal operator of $h$ in terms of the (unweighted) proximal operator of another function:
\begin{lemma}[\citealt{becker2019quasinewton}]
    \label{lem:weighted-prox-to-prox}
    Let $\mA$ be a positive definite matrix. Then we have
    \begin{equation}
        \prox_{\alpha h}^{\mA}(\vx) = \mA^{-1/2} \prox_{\alpha h \circ \mA^{-1/2}}(\mA^{1/2}\vx).
    \end{equation}
\end{lemma}
This lemma is valid for any positive definite matrix $\mA$; in particular, \citet{becker2019quasinewton} makes use of it to derive proximal quasi-Newton algorithms based on low-rank approximations of the Hessian. Here, we will only use it in the limited case where the matrix $\mD$ is diagonal (and positive definite).

\subsection{Weighted proximal operator of $\ell_{2}$}
\label{app:proof-weighted-prox-l2}
Based on \cref{lem:weighted-prox-to-prox}, in order to compute the weighted proximal operator of $h(\vx) = \lambda\|\vx\|_{2}$, it is sufficient to compute the (unweighted) proximal operator $h\circ \mA^{-1/2}$. The following proposition shows how to compute this operator, up to a change of variable (to simplify the notation), for a diagonal positive definite matrix $\mD$. \cref{prop:prox-l2-unweighted} is based on a result derived in \citep{li2020proxl2}; we include the proof here for completeness.

\begin{proposition}[\citealp{li2020proxl2}]
    \label{prop:prox-l2-unweighted}
    Let $\mD = \diag(d_{1}, \ldots, d_{n})$ be a positive definite diagonal matrix (i.e. $d_{i} > 0$ for all $i$), and $h(\vx) = \lambda \|\vx\|_{2}$ the $\ell_{2}$ penalty. The proximal operator of $h \circ \mD$ is given by
    \begin{equation}
        \big[\prox_{\alpha h \circ \mD}(\vx)\big]_{i} = \left\{\begin{array}{cl}
            \dfrac{\theta^{\star}x_{i}}{\theta^{\star} + \alpha \lambda d_{i}^{2}} & \textrm{if $\|\mD^{-1}\vx\|_{2} > \alpha \lambda$} \\[1em]
            0 & \textrm{otherwise,}
        \end{array}\right.
    \end{equation}
    where $\theta^{\star} > 0$ is the unique positive solution of
    \begin{equation}
        G(\theta^{\star}) \triangleq \sum_{i=1}^{n}\left[\frac{d_{i}x_{i}}{\theta^{\star} + \alpha\lambda d_{i}^{2}}\right]^{2} = 1.
        \label{eq:condition-theta}
    \end{equation}
\end{proposition}

\begin{proof}
Let us first recall the definition of the proximal operator of $h \circ \mD$
\begin{equation*}
    \prox_{\alpha h \circ \mD}(\vx) \triangleq \argmin_{\vz \in \sR^{n}}\underbrace{\frac{1}{2}\|\vz - \vx\|_{2}^{2} + \alpha\lambda \|\mD\vz\|_{2}}_{\triangleq\,F(\vz)}.
\end{equation*}

The function $F$ is a convex function, whose minimizers $\vz^{\star}$ characterize the proximal operator of interest. We start with the special case where $\vz^{\star} = \vzero$; by definition of a global minimizer, we have
\begin{align*}
    \forall \vz \in \sR^{n},\ \ F(\vz) \geq F(\vzero) & & \Leftrightarrow & & & \forall \vz \in \sR^{n},\ \  \frac{1}{2}\|\vz\|_{2}^{2} + \alpha \lambda \|\mD\vz\|_{2} \geq \langle \vz, \vx \rangle\\
    \eqcomment{change of variable $\vz \rightarrow \mD^{-1}\vy$}{\ }& & \Leftrightarrow & & & \forall \vy \in \sR^{n},\ \  \frac{1}{2}\|\mD^{-1}\vy\|_{2}^{2} + \alpha \lambda \|\vy\|_{2} \geq \langle \mD^{-1}\vy, \vx \rangle\\
    \eqcomment{$\mD^{-1}$ is symmetric}{\ }& & \Leftrightarrow & & & \forall \vy \in \sR^{n},\ \ \frac{1}{2}\|\mD^{-1}\vy\|_{2}^{2} + \alpha \lambda \|\vy\|_{2} \geq \langle \vy, \mD^{-1}\vx\rangle
\end{align*}

If we introduce a new function $H(\vy) \triangleq 1/2\|\mD^{-1}\vy\|_{2}^{2} + \alpha \lambda \|\vy\|_{2}$, by definition of a subgradient, the last inequality is equivalent to $\mD^{-1}\vx \in \partial H(\vzero)$, where $\partial H(\vzero)$ is the subdifferential of $H$ at $\vzero$:

\begin{equation*}
    \partial H(\vzero) = \{\vy \in \sR^{n}\mid \|\vy\|_{2} \leq \alpha \lambda\}.
\end{equation*}

This shows that $\vzero$ is a global minimizer of $F$ iff $\|\mD^{-1}\vx\|_{2} \leq \alpha \lambda$; in other words,
\begin{equation*}
    \prox_{\alpha h\circ \mD}(\vx) = \vzero \qquad \Leftrightarrow \qquad \|\mD^{-1}\vx\|_{2} \leq \alpha \lambda.
\end{equation*}

Now if the minimizer $\vz^{\star} \neq \vzero$ of $F$ is non-zero, we can use the necessary and sufficient first-order condition of optimality for $F$ ($F$ being convex and differentiable at $\vz^{\star}$ for $\vz^{\star} \neq \vzero$) to obtain
\begin{equation*}
    \vz^{\star} - \vx + \alpha \lambda \frac{\mD^{2}\vz^{\star}}{\|\mD\vz^{\star}\|_{2}} = \vzero \qquad \Leftrightarrow \qquad \vz^{\star} = \left[\mI + \alpha\lambda \frac{\mD^{2}}{\|\mD\vz^{\star}\|_{2}}\right]^{-1}\vx
\end{equation*}

Let us call $\theta^{\star} \triangleq \|\mD\vz^{\star}\|_{2} > 0$. We can rewrite the coordinates $z_{i}^{\star}$ of $\vz^{\star}$ more explicitly, as a function of $x_{i}$ and $\theta^{\star}$
\begin{equation*}
    \big[\prox_{\alpha h \circ \mD}(\vx)\big]_{i} \triangleq z_{i}^{\star} = \frac{\theta^{\star}x_{i}}{\theta^{\star} + \alpha\lambda d_{i}^{2}}.
\end{equation*}

Although we have found an expression for $\vz^{\star}$, the constant $\theta^{\star}$ remains a function of $\vz^{\star}$. This introduces some constraints on $\theta^{\star}$, which in turn will introduce constraints on $\vx$ to guarantee the existence of $\theta^{\star}$. For example, we have
\begin{equation*}
    \theta^{\star 2} = \|\mD\vz^{\star}\|_{2}^{2} = \sum_{i=1}^{n} d_{i}^{2}z_{i}^{\star 2} = \theta^{\star 2}\sum_{i=1}^{n} \left[\frac{d_{i}x_{i}}{\theta^{\star} + \alpha\lambda d_{i}^{2}}\right]^{2} \qquad \Leftrightarrow \qquad G(\theta^{\star}) = 1,
\end{equation*}

which is the condition on $\theta^{\star}$ in \cref{eq:condition-theta}. Moreover, the function $G$ is convex and monotonically decreasing on the positive line $\sR_{+}$. This implies that a solution $\theta^{\star} > 0$ of $G(\theta^{\star}) = 1$ exists (and is unique) if and only if $G(0) > 1$
\begin{equation*}
    G(0) = \frac{\|\mD^{-1}\vx\|_{2}^{2}}{(\alpha\lambda)^{2}} > 1 \qquad \Leftrightarrow \qquad \|\mD^{-1}\vx\|_{2} > \alpha \lambda.
\end{equation*}
In other words, we can conclude that
\begin{equation*}
    \big[\prox_{\alpha h \circ \mD}(\vx)\big]_{i} = \frac{\theta^{\star}x_{i}}{\theta^{\star} + \alpha \lambda d_{i}^{2}} \qquad \Leftrightarrow \qquad \|\mD^{-1}\vx\|_{2} > \alpha \lambda.
\end{equation*}
\end{proof}
Finally, \cref{thm:weighted-prox-l2} is a corollary of the above proposition. Let's first recall the theorem:
\begingroup
\def\thetheorem{\ref{thm:weighted-prox-l2}}
\begin{theorem}[Weighted proximal operator of $\ell_{2}$]
  Let $\mD = \diag(d_{1}, \ldots, d_{n})$ be a positive definite diagonal matrix (i.e. $d_{i} > 0$ for all $i$), and $h(\vx) = \lambda \|\vx\|_{2}$ the $\ell_{2}$ penalty. The weighted proximal operator of $h$ is given by
  \begin{equation*}
    \big[\prox_{\alpha h}^{\mD}(\vx)\big]_{i} = \left\{\begin{array}{cl}
        \dfrac{d_{i}\theta^{\star}x_{i}}{d_{i}\theta^{\star} + \alpha \lambda} & \textrm{if $\|\mD\vx\|_{2} > \alpha \lambda$}\\[1em]
        0 & \textrm{otherwise},
    \end{array}\right.
  \end{equation*}
  where $\theta^{\star} > 0$ is the unique positive solution of
  \begin{equation*}
    \sum_{i=1}^{n}\left[\frac{d_{i}x_{i}}{d_{i}\theta^{\star} + \alpha\lambda}\right]^{2} = 1.
  \end{equation*}
\end{theorem}
\addtocounter{theorem}{-1}
\endgroup
\begin{proof}
    We get the expected result using the changes of variable $\mD \rightarrow \mD^{-1/2}$ and $\vx \rightarrow \mD^{1/2}\vx$ in \cref{prop:prox-l2-unweighted}, together with \cref{lem:weighted-prox-to-prox} (pre-multiplication by $\mD^{-1/2}$). 
\end{proof}



\subsection{Approximation of the weighted proximal operator of $\ell_{2}$}
\label{app:proof-approx-weighted-prox-l2}
In this section, we prove \cref{prop:weighted-prox-l2-approx}, which is recalled here:
\begingroup
\def\theproposition{\ref{prop:weighted-prox-l2-approx}}
\begin{proposition}
    Let $\varepsilon > 0$, and $\tilde{\theta} > 0$ such that $|\tilde{\theta} - \theta^{\star}| \leq \varepsilon$ (recall that $\theta^{\star}$ is defined as the unique positive solution of \cref{eq:weighted-prox-l2-theta}). Let $\widetilde{\prox}^{\mD}_{\alpha h}(\vx)$ be an approximation of the weighted proximal operator of $\ell_{2}$ (i.e. $h(\vx) = \lambda\|\vx\|_{2}$), replacing $\theta^{\star}$ by $\tilde{\theta}$ in \cref{eq:weighted-prox-l2}. Then we have for all $\vx \in \sR^{n}$:
    \begin{equation}
        \big\|\widetilde{\prox}_{\alpha h}^{\mD}(\vx) - \prox_{\alpha h}^{\mD}(\vx)\big\|_{2} \leq \varepsilon.
    \end{equation}
    \label{prop:approx-weighted-prox-l2}
\end{proposition}
\addtocounter{proposition}{-1}
\endgroup

\begin{proof}
Let $\mD = \diag(d_{1}, \ldots, d_{n})$ be a positive definite diagonal matrix, and $d_{\min} > 0$ the smallest value of the diagonal of $\mD$. For all $\vx \in \sR^{n}$
\begin{align*}
    \big\|\widetilde{\prox}_{\alpha h}^{\mD}(\vx) - \prox_{\alpha h}^{\mD}(\vx)\big\|_{2}^{2} &\leq \sum_{i=1}^{n}\left[\frac{d_{i}\tilde{\theta}x_{i}}{d_{i}\tilde{\theta} + \alpha \lambda} - \frac{d_{i}\theta^{\star}x_{i}}{d_{i}\theta^{\star} + \alpha \lambda}\right]^{2}\\
    & = \sum_{i=1}^{n}\left[\frac{\alpha\lambda d_{i}x_{i}(\tilde{\theta} - \theta^{\star})}{(d_{i}\tilde{\theta} + \alpha\lambda)(d_{i}\theta^{\star} + \alpha\lambda)}\right]^{2}\\
    &\leq (\tilde{\theta} - \theta^{\star})^{2}\underbrace{\vphantom{\sum_{i=1}^{n}}\left[\frac{\alpha\lambda}{d_{\min}\tilde{\theta} + \alpha \lambda}\right]^{2}}_{\leq\,1}\underbrace{\sum_{i=1}^{n}\left[\frac{d_{i}x_{i}}{d_{i}\theta^{\star} + \alpha\lambda}\right]^{2}}_{=\,1\ \textrm{(\cref{eq:weighted-prox-l2-theta}})}\\
    &\leq (\tilde{\theta} - \theta^{\star})^{2} \leq \varepsilon^{2}
\end{align*}
\end{proof}

The following proposition also gives bounds on $\theta^{\star}$ defined in \cref{eq:weighted-prox-l2-theta}, to narrow down the search space for the numerical solver (such as the Newton-Raphson algorithm, see \cref{app:newton-raphson} for details).
\begin{proposition}[Bounds on $\theta^{\star}$ for the $\ell_{1}/\ell_{2}$ penalty]
    Let $\mD = \diag(d_{1}, \ldots, d_{n})$ be a positive definite diagonal matrix, with $d_{\min}$ and $d_{\max}$ being respectively the smallest and largest values of the diagonal of $\mD$. Let $\theta^{\star}$ defined by \cref{eq:weighted-prox-l2-theta}. Then for all $\vx$ such that $\|\mD\vx\|_{2} > \alpha \lambda$:
    \begin{equation}
        0 < \frac{\|\mD\vx\|_{2} - \alpha \lambda}{d_{\max}} \leq \theta^{\star} \leq \frac{\|\mD\vx\|_{2} - \alpha\lambda}{d_{\min}}.
    \end{equation}
    \label{prop:bounds-theta-weighted-prox-l2}
\end{proposition}
\begin{proof}
Using the fact that for all $i$ we have $d_{\min} \leq d_{i} \leq d_{\max}$, we get the following inequalities:
\begin{equation*}
    \left[\frac{\|\mD\vx\|_{2}}{d_{\max}\theta^{\star} + \alpha\lambda}\right]^{2} \leq \underbrace{\sum_{i=1}^{n}\left[\frac{d_{i}x_{i}}{d_{i}\theta^{\star} + \alpha\lambda}\right]^{2}}_{=\,1\ \textrm{(\cref{eq:weighted-prox-l2-theta})}} \leq \left[\frac{\|\mD\vx\|_{2}}{d_{\min}\theta^{\star} + \alpha\lambda}\right]^{2}
\end{equation*}
These two inequalities give us the expected bounds on $\theta^{\star}$:
\begin{align*}
    \frac{\|\mD\vx\|_{2}}{d_{\max}\theta^{\star} + \alpha\lambda} &\leq 1 &\Leftrightarrow&& \frac{\|\mD\vx\|_{2} - \alpha\lambda}{d_{\max}} & \leq \theta^{\star}\\
    \frac{\|\mD\vx\|_{2}}{d_{\min}\theta^{\star} + \alpha\lambda} &\geq 1 &\Leftrightarrow&& \frac{\|\mD\vx\|_{2} - \alpha\lambda}{d_{\min}} & \geq \theta^{\star}
\end{align*}
\end{proof}

\subsection{Weighted proximal operator of MCP / $\ell_{2}$}
\label{app:proof-weighted-prox-mcp-l2}
Similar to \cref{app:proof-weighted-prox-l2}, we start by giving the expression of the proximal operator of $h \circ \mA^{-1/2}$ (up to a change of variable, for simplicity), in order to prove \cref{thm:weighted-prox-mcp-l2} via \cref{lem:weighted-prox-to-prox}.
\begin{proposition}
    \label{prop:prox-mcp-l2-unweighted}
    Let $\mD = \diag(d_{1}, \ldots, d_{n})$ be a positive definite diagonal matrix (i.e. $d_{i} > 0$ for all $i$), and $h(\vx\,;\beta, \lambda)$ the MCP / $\ell_{2}$ penalty defined in \cref{eq:mcp-l2}. Suppose that $\alpha$ and $\beta$ satisfy $\beta > \alpha d_{\max}^{2}$, where $d_{\max}$ is the largest value of the diagonal of $\mD$. The proximal operator of $h \circ \mD$ is given by
    \begin{equation}
        \big[\prox_{\alpha h \circ \mD}(\vx)\big]_{i} = \left\{\begin{array}{cl}
            x_{i} & \textrm{if $\|\mD\vx\|_{2} > \beta\lambda$} \\[0.8em]
            \dfrac{\beta\theta^{\star}x_{i}}{\beta\theta^{\star} + \alpha d_{i}^{2}(\beta\lambda - \theta^{\star})} & \textrm{\begin{tabular}{@{}l}if $\|\mD\vx\|_{2} \leq \beta \lambda$\\[0.3em] and $\|\mD^{-1}\vx\|_{2} > \alpha \lambda$\end{tabular}}\\[1.5em]
            0 & \textrm{\begin{tabular}{@{}l}if $\|\mD\vx\|_{2} \leq \beta \lambda$\\[0.3em] and $\|\mD^{-1}\vx\|_{2} \leq \alpha \lambda$,\end{tabular}}
        \end{array}\right.
    \end{equation}
    
    where $\theta^{\star} > 0$ is the unique positive solution of
    \begin{equation}
        G(\theta^{\star}) \triangleq \beta^{2}\sum_{i=1}^{n}\left[\frac{d_{i}x_{i}}{\theta^{\star}(\beta - \alpha d_{i}^{2}) + \alpha\beta\lambda d_{i}^{2}}\right]^{2} = 1.
        \label{eq:condition-theta-mcp}
    \end{equation}
\end{proposition}

\begin{proof}
Let us first recall the definition of the proximal operator of $h \circ \mD$
\begin{equation*}
    \prox_{\alpha h\circ \mD}(\vx) \triangleq \argmin_{\vz\in\sR^{n}}\underbrace{\frac{1}{2}\|\vz - \vx\|_{2}^{2} + \alpha h(\mD\vz\,;\beta, \lambda)}_{\triangleq\,F(\vz)}.
\end{equation*}

Let us first consider the case where a global minimizer $\vz^{\star}$ of $F$ satisfies $\|\mD\vz^{\star}\|_{2} > \beta\lambda$. By definition of the MCP / $\ell_{2}$ penalty in \cref{eq:mcp-l2}, this means that
\begin{equation*}
    F(\vz^{\star}) = \argmin_{\vz\in\sR^{n}}\frac{1}{2}\|\vz - \vx\|_{2}^{2} + \alpha\frac{\beta\lambda^{2}}{2}.
\end{equation*}

Since the condition $\|\mD\vz^{\star}\|_{2} > \beta\lambda$ implies that $\vz^{\star} \neq \vzero$, we can use the necessary and sufficient first-order condition of optimality for $F$ (which is convex and differentiable at $\vz^{\star} \neq \vzero$) to get $\vz^{\star} = \vx$. In other words, we have
\begin{equation*}
    \prox_{\alpha h \circ \mD}(\vx) = \vx \qquad \Leftrightarrow \qquad \|\mD\vx\|_{2} > \beta\lambda.
\end{equation*}

Otherwise, if a global minimizer $\vz^{\star}$ of $F$ satisfies $\|\mD\vz^{\star}\|_{2} \leq \beta\lambda$, then again by definition of $h$:
\begin{equation*}
    F(\vz^{\star}) = \argmin_{\vz\in\sR^{n}}\frac{1}{2}\|\vz - \vx\|_{2}^{2} + \alpha\left(\lambda\|\mD\vz\|_{2} - \frac{\|\mD\vz\|_{2}^{2}}{2\beta}\right).
\end{equation*}

Let us first consider the special case where $\vzero$ is a minimizer of $F$. Similar to the proof of \cref{prop:prox-l2-unweighted}, by definition of $\vzero$ being a global minimizer of $F$, we have
\begin{align*}
    \forall \vz \in \sR^{n},\ \ F(\vz) \geq F(\vzero) & & \Leftrightarrow & & & \forall \vz \in \sR^{n},\ \  \frac{1}{2}\|\vz\|_{2}^{2} + \alpha \left(\lambda \|\mD\vz\|_{2} - \frac{\|\mD\vz\|_{2}^{2}}{2\beta}\right) \geq \langle \vz, \vx \rangle\\
    & & \Leftrightarrow & & & \forall \vy \in \sR^{n},\ \ \frac{1}{2}\|\mD^{-1}\vy\|_{2}^{2} - \alpha\frac{\|\vy\|_{2}^{2}}{2\beta} + \alpha\lambda \|\vy\|_{2} \geq \langle \vy, \mD^{-1}\vx\rangle.
\end{align*}

The last inequality is equivalent to $\mD^{-1}\vx \in \partial H(\vzero)$, where $\partial H(\vzero)$ is the subdifferential of the function $H$ defined by:
\begin{align*}
    & H(\vy) \triangleq \frac{1}{2}\|\mD^{-1}\vy\|_{2}^{2} - \alpha \frac{\|\vy\|_{2}^{2}}{2\beta} + \alpha\lambda\|\vy\|_{2}\\
    \textrm{and}\ \ & \partial H(\vzero) = \{\vy \in \sR^{n}\mid \|\vy\|_{2}\leq \alpha\lambda\}.
\end{align*}

Note that the function $H$ is convex thanks to the condition $\beta > \alpha d_{\max}^{2}$, and therefore its subdifferential is well defined. This proves that $\vzero$ is a global minimizer of $F$ iff $\|\mD^{-1}\vx\|_{2} \leq \alpha \lambda$ (and $\|\mD\vx\|_{2} \leq \beta \lambda$, given the first case considered above). In other words,
\begin{equation*}
    \prox_{\alpha h\circ \mD}(\vx) = \vzero \qquad \Leftrightarrow \qquad \|\mD\vx\|_{2} \leq \beta\lambda\ \  \textrm{and}\ \ \|\mD^{-1}\vx\|_{2} \leq \alpha\lambda.
\end{equation*}

If the minimizer $\vz^{\star} \neq \vzero$ of $F$ is non-zero and satisfies $\|\mD\vz^{\star}\|_{2} \leq \beta\lambda$, then we can again use the necessary and sufficient first-order condition of optimality for $F$ to get
\begin{equation*}
    \vz^{\star} - \vx + \alpha\left(\frac{\lambda \mD^{2}\vz^{\star}}{\|\mD\vz^{\star}\|_{2}} - \frac{\mD^{2}\vz^{\star}}{\beta}\right) = \vzero \qquad \Leftrightarrow \qquad \vz^{\star} = \left[\mI + \alpha\left(\frac{\lambda}{\|\mD\vz^{\star}\|_{2}} - \frac{1}{\beta}\right)\mD^{2}\right]^{-1}\vx.
\end{equation*}

Let us call $\theta^{\star} \triangleq \|\mD\vz^{\star}\|_{2} > 0$. We can rewrite the coordinates $z_{i}^{\star}$ of $\vz^{\star}$ more explicitly, as a function of $x_{i}$ and $\theta^{\star}$
\begin{equation*}
    \big[\prox_{\alpha h\circ \mD}(\vx)\big]_{i} \triangleq z_{i}^{\star} = \frac{\beta\theta^{\star}x_{i}}{\beta\theta^{\star} + \alpha d_{i}^{2}(\beta\lambda - \theta^{\star})}.
\end{equation*}

Similar to the proof of \cref{prop:prox-l2-unweighted}, the constraints on $\theta^{\star}$ induce constraints on $\vx$ to guarantee the existence of $\theta^{\star}$. For example, we have
\begin{equation*}
    \theta^{\star 2} = \|\mD\vz^{\star}\|_{2}^{2} = \sum_{i=1}^{n}d_{i}^{2}z_{i}^{\star 2} = \beta^{2}\theta^{\star 2}\sum_{i=1}^{n}\frac{d_{i}^{2}x_{i}^{2}}{\big(\theta^{\star}(\beta - \alpha d_{i}^{2}) + \alpha\beta\lambda d_{i}^{2}\big)^{2}} \qquad \Leftrightarrow \qquad G(\theta^{\star}) = 1,
\end{equation*}
which is the condition on $\theta^{\star}$ in \cref{eq:condition-theta-mcp}. Moreover, since we made the assumption that $\beta > \alpha d_{\max}^{2}$, the function $G$ is convex and monotonically decreasing on the positive line $\mathbb{R}_{+}$. This implies that a solution $\theta^{\star} \in (0, \beta \lambda]$ of $G(\theta^{\star}) = 1$ exists (recall that we are in the case where $\|\mD\vz^{\star}\|_{2} \leq \beta\lambda$), and is unique, if and only if $G(0) > 1 \geq G(\beta \lambda)$. In other words, this implies
\begin{align*}
    G(0) = \frac{\|\mD^{-1}\vx\|_{2}^{2}}{(\alpha\lambda)^{2}} &> 1 & \Leftrightarrow & & \|\mD^{-1}\vx\|_{2} &> \alpha\lambda\\
    \textrm{and}\ \ G(\beta\lambda) = \frac{\|\mD\vx\|_{2}^{2}}{(\beta\lambda)^{2}} &\leq 1 & \Leftrightarrow & & \|\mD\vx\|_{2} &\leq \beta\lambda.
\end{align*}
\end{proof}

Finally, \cref{thm:weighted-prox-mcp-l2} is a corollary of the above proposition. Let's first recall the theorem:
\begingroup
\def\thetheorem{\ref{thm:weighted-prox-mcp-l2}}
\begin{theorem}[Weighted proximal operator of MCP / $\ell_{2}$]
  Let $\mD = \diag(d_{1}, \ldots, d_{n})$ be a positive definite diagonal matrix (i.e. $d_{i} > 0$ for all $i$), and $h(\vx) = \MCP(\|\vx\|_{2})$ the MCP / $\ell_{2}$ penalty. Suppose that $\alpha$ and $\beta$ satisfy $\alpha < \beta d_{\min}$, where $d_{\min}$ is the smallest value of the diagonal of $\mD$. The weighted proximal operator of $h$ is given by
  \begin{equation*}
      \big[\prox_{\alpha h}^{\mD}(\vx)\big]_{i} = \left\{\begin{array}{cl}
            x_{i} & \textrm{if $\|\vx\|_{2} > \beta\lambda$} \\[0.8em]
            \dfrac{d_{i}\beta\theta^{\star}x_{i}}{(d_{i}\beta - \alpha)\theta^{\star} + \alpha \beta\lambda} & \textrm{\begin{tabular}{@{}l}if $\|\vx\|_{2} \leq \beta \lambda$\\[0.4em] and $\|\mD\vx\|_{2} > \alpha \lambda$\end{tabular}}\\[1.5em]
            0 & \textrm{\begin{tabular}{@{}l}if $\|\vx\|_{2} \leq \beta \lambda$\\[0.4em] and $\|\mD\vx\|_{2} \leq \alpha \lambda$,\end{tabular}}
        \end{array}\right.
  \end{equation*}
  where $\theta^{\star} > 0$ is the unique positive solution of
  \begin{equation*}
      \beta^{2}\sum_{i=1}^{n}\left[\frac{d_{i}x_{i}}{(d_{i}\beta - \alpha)\theta^{\star} + \alpha\beta\lambda}\right]^{2} = 1.
  \end{equation*}
\end{theorem}
\addtocounter{theorem}{-1}
\endgroup
\begin{proof}
We get the expected result using the changes of variables $\mD \rightarrow \mD^{-1/2}$ and $\vx \rightarrow \mD^{1/2}\vx$ in \cref{prop:prox-mcp-l2-unweighted}, together with \cref{lem:weighted-prox-to-prox} (pre-multiplication by $\mD^{-1/2}).$
\end{proof}

\subsection{Approximation of the weighted proximal operator of MCP / $\ell_{2}$}
\label{app:proof-approx-weighted-prox-mcp-l2}
Similar to \cref{prop:weighted-prox-l2-approx}, the following proposition shows that the error on the weighted proximal operator of MCP / $\ell_{2}$ induced by the numerical solver to find $\theta^{\star}$ is controlled by the error of the numerical solver.
\begin{proposition}
    Let $\varepsilon > 0$, and $\tilde{\theta} > 0$ such that $|\tilde{\theta} - \theta^{\star}| \leq \varepsilon$ (recall that $\theta^{\star}$ is defined as the unique positive solution of \cref{eq:weighted-prox-mcp-l2-theta}). Let $\widetilde{\prox}^{\mD}_{\alpha h}(\vx)$ be an approximation of the weighted proximal operator of MCP / $\ell_{2}$ (i.e. $h(\vx) = \MCP(\|\vx\|_{2};\beta, \lambda)$), replacing $\theta^{\star}$ by $\tilde{\theta}$ in \cref{eq:weighted-prox-mcp-l2}.
    
    Suppose that $\alpha$ and $\beta$ satisfy $\alpha < d_{\min}\beta$, where $d_{\min}$ is the smallest value of the diagonal of $\mD$ (recall that $\mD$ is a positive definite diagonal matrix). Then we have for all $\vx \in \sR^{n}$:
    \begin{equation}
        \big\|\widetilde{\prox}_{\alpha h}^{\mD}(\vx) - \prox_{\alpha h}^{\mD}(\vx)\big\|_{2} \leq \varepsilon.
    \end{equation}
\end{proposition}

\begin{proof}
The proof is similar to the one of \cref{prop:approx-weighted-prox-l2} in \cref{app:proof-approx-weighted-prox-l2}. Let $\mD = \diag(d_{1}, \ldots, d_{n})$ be a positive definite diagonal matrix, and $d_{\min} > 0$ the smallest value of the diagonal of $\mD$. For all $\vx \in \sR^{n}$
\begin{align*}
    \big\|\widetilde{\prox}_{\alpha h}^{\mD}(\vx) - \prox_{\alpha h}^{\mD}(\vx)\big\|_{2}^{2} & \leq \sum_{i=1}^{n} \left[\frac{d_{i}\beta\tilde{\theta}x_{i}}{(d_{i}\beta - \alpha)\tilde{\theta} + \alpha\beta\lambda} - \frac{d_{i}\beta\theta^{\star}x_{i}}{(d_{i}\beta - \alpha)\theta^{\star} + \alpha\beta\lambda}\right]^{2}\\
    & = \sum_{i=1}^{n}\left[\frac{\alpha\beta^{2}\lambda d_{i}x_{i}(\tilde{\theta} - \theta^{\star})}{((d_{i}\beta - \alpha)\tilde{\theta} + \alpha\beta\lambda)((d_{i}\beta - \alpha)\theta^{\star} + \alpha\beta\lambda)}\right]^{2}\\
    &\leq (\tilde{\theta} - \theta^{\star})^{2}\underbrace{\vphantom{\sum_{i=1}^{2}}\left[\frac{\alpha\beta\lambda}{(d_{\min}\beta - \alpha)\tilde{\theta} + \alpha\beta\lambda}\right]^{2}}_{\leq\,1}\underbrace{\beta^{2}\sum_{i=1}^{n}\left[\frac{d_{i}x_{i}}{(d_{i}\beta - \alpha)\theta^{\star} + \alpha\beta\lambda}\right]^{2}}_{=\,1\ \textrm{(\cref{eq:weighted-prox-mcp-l2-theta})}}\\
    &\leq (\tilde{\theta} - \theta^{\star})^{2} \leq \varepsilon^{2}
\end{align*}
\end{proof}
Finally, the following proposition gives bounds on $\theta^{\star}$ defined in \cref{eq:weighted-prox-mcp-l2-theta}, to narrow down the search space for the numerical solver (e.g. Newton-Raphson algorithm).
\begin{proposition}[Bounds on $\theta^{\star}$ for the MCP / $\ell_{2}$ penalty]
    Let $\mD = \diag(d_{1}, \ldots, d_{n})$ be a positive definite diagonal matrix, with $d_{\min}$ and $d_{\max}$ being respectively the smallest and largest values of the diagonal of $\mD$. Let $\theta^{\star}$ defined by \cref{eq:weighted-prox-mcp-l2-theta}, and suppose that $\alpha$ and $\beta$ satisfy $\alpha < d_{\min}\beta$. Then for all $\vx$ such that $\|\vx\|_{2} \leq \beta\lambda$ and $\|\mD\vx\|_{2} > \alpha\lambda$
    \begin{equation}
        0 < \beta\frac{\|\mD\vx\|_{2} - \alpha\lambda}{d_{\max}\beta - \alpha} \leq \theta^{\star} \leq \beta\frac{\|\mD\vx\|_{2} - \alpha\lambda}{d_{\min}\beta - \alpha}.
    \end{equation}
    \label{prop:bounds-theta-weighted-prox-mcp-l2}
\end{proposition}

\begin{proof}
The proof is similar to the one of \cref{prop:bounds-theta-weighted-prox-l2} in \cref{app:proof-approx-weighted-prox-l2}. Using the fact that for all $i$ we have $d_{\min} \leq d_{i} \leq d_{\max}$, we get the following inequalities:
\begin{equation*}
    \left[\frac{\beta\|\mD\vx\|_{2}}{(d_{\max}\beta - \alpha)\theta^{\star} + \alpha\beta\lambda}\right]^{2} \leq \underbrace{\beta^{2}\sum_{i=1}^{n}\left[\frac{d_{i}x_{i}}{(d_{i}\beta - \alpha)\theta^{\star} + \alpha\beta\lambda}\right]^{2}}_{=\,1\ \textrm{(\cref{eq:weighted-prox-mcp-l2-theta})}} \leq \left[\frac{\beta\|\mD\vx\|_{2}}{(d_{\min}\beta - \alpha)\theta^{\star} + \alpha\beta\lambda}\right]^{2}
\end{equation*}
These two inequalities give us the expected bounds on $\theta^{\star}$:
\begin{align*}
    \frac{\beta\|\mD\vx\|_{2}}{(d_{\max}\beta - \alpha)\theta^{\star} + \alpha\beta\lambda} &\leq 1 &\Leftrightarrow&& \beta\frac{\|\mD\vx\|_{2} - \alpha\lambda}{d_{\max}\beta - \alpha} & \leq \theta^{\star}\\
    \frac{\beta\|\mD\vx\|_{2}}{(d_{\min}\beta - \alpha)\theta^{\star} + \alpha\beta\lambda} &\geq 1 &\Leftrightarrow&& \beta\frac{\|\mD\vx\|_{2} - \alpha\lambda}{d_{\min}\beta - \alpha} & \geq \theta^{\star}
\end{align*}
\end{proof}

\section{Proof of the convergence analysis}
\label{app:proof-convergence-analysis}
In this section, we prove the main theorem of convergence from \cref{sec:convergence-analysis}, extending the analysis from \citet{xu2019analysisprox,yun2020proxgen} to inexact proximal gradient methods, under \cref{hyp:prox-approximation} on the approximation of the weighted proximal operator. Before getting the the proof of \cref{thm:convergence-analysis}, we first recall the definition of the Frechet subdifferential, which plays a central role in the statement of convergence for non-differentiable and non-convex functions.
\begin{definition}[Frechet subdifferential; \citealt{rockafellar2009variational}]
    Let $F: \sR^{N} \rightarrow \overline{\sR}$ be a function (possibly non-convex), and a point $\overline{\vx}$ such that $F(\overline{\vx})$ is finite. The \emph{Frechet subdifferential} is the set $\widehat{\partial}F$ of regular subgradients, satisfying:
    \begin{equation}
        \widehat{\partial}F(\overline{\vx}) = \bigg\{\vv \in \sR^{N} \;\Big|\; \liminf_{\substack{\vx \rightarrow \overline{\vx}\\\vx \neq \overline{\vx}}} \frac{F(\vx) - F(\overline{\vx}) - \langle \vv, \vx - \overline{\vx}\rangle}{\vx - \overline{\vx}} \geq 0\bigg\}.
    \end{equation}
    \label{def:frechet-subdifferential}
\end{definition}
Recall that $F(\vx) \triangleq \E_{\xi}[f(\vx;\xi)] + h(\vx)$. We are now ready to prove \cref{thm:convergence-analysis}:
\begingroup
\def\thetheorem{\ref{thm:convergence-analysis}}
\begin{theorem}
    Suppose that \cref{hyp:proxgen-assumptions,hyp:prox-approximation} are satisfied. If we run \cref{alg:minimization-composite-loss} with a non-increasing step-size $\alpha_{t}$, such that $\alpha_{0} < \delta/2L$, then the output $\vx_{R}$ of \cref{alg:minimization-composite-loss}, where $R$ is sampled uniformly in $\{1, \ldots, T\}$, satisfies
    \begin{equation*}
        \displaystyle \E_{R}\big[\dist\!\big(\vzero, \widehat{\partial}F(\vx_{R})\big)^{2}\big] \leq \frac{C_{1}}{T}\sum_{t=0}^{T-1}\|\vg_{t} - \nabla f(\vx_{t})\|_{2}^{2} + \frac{C_{2}\Delta}{T} + \frac{C_{3}}{T},
    \end{equation*}
    with $C_{1}$, $C_{2}$, and $C_{3}$ positive constants independent of $T$. Here $\dist(\vz, S)$ is the distance of a set $S$ to a point $\vz$, defined as the minimal distance of any point in $S$ to $\vz$.
\end{theorem}
\addtocounter{theorem}{-1}
\endgroup
\begin{proof}
The proof follows the same steps as the proof of the non-asymptotic convergence analysis from \citet{yun2020proxgen} and \citet{xu2019analysisprox}. We detail the full proof here, including the steps from \citep{yun2020proxgen}, for completeness. Recall from \cref{hyp:prox-approximation} the notation
\begin{equation*}
    \vx_{t+1}^{\star} = \prox_{\alpha h}^{\mD_{t}}\!\big(\vx_{t} - \alpha \mD^{-1}\vm_{t}\big).
\end{equation*}
\paragraph{Difference of consecutive iterates} Let us first consider the case of $t$ fixed, where \cref{hyp:prox-approximation} (2) is satisfied; that is, we assume that $\vx_{t+1}$ is an $\varepsilon_{t+1}$-approximation of the true proximal update $\vx_{t+1}^{\star}$, and $h$ is $L'$-smooth in a $\varepsilon_{t+1}$-ball around $\vx_{t+1}^{\star}$. By definition of the weighted proximal operator as a minimizer in \cref{eq:weighted-proximal-operator}, we have $\vx_{t+1}^{\star}$ satisfying
\begin{align}
    & \langle \vm_{t}, \vx_{t+1}^{\star} \rangle + h(\vx_{t+1}^{\star}) + \frac{1}{2\alpha_{t}}\|\vx_{t+1}^{\star} - \vx_{t}\|^{2}_{\mD_{t}} \leq \langle \vm_{t}, \vx_{t} \rangle + h(\vx_{t})\nonumber\\
    \Leftrightarrow\qquad & \langle \vm_{t}, \vx_{t+1}^{\star} - \vx_{t} \rangle + h(\vx_{t+1}^{\star}) + \frac{1}{2\alpha_{t}}\|\vx_{t+1}^{\star} - \vx_{t}\|^{2}_{\mD_{t}} \leq h(\vx_{t})\label{eq:proof-prox-minimizer}
\end{align}
Moreover since we assumed that $h$ is $L'$-smooth in a $\varepsilon_{t+1}$-ball $\gB$ around $\vx_{t+1}^{\star}$, and $\vx_{t+1} \in \gB$:
\begin{equation}
    h(\vx_{t+1}) \leq h(\vx_{t+1}^{\star}) + \langle \nabla h(\vx_{t+1}^{\star}), \vx_{t+1} - \vx_{t+1}^{\star} \rangle + \frac{L'}{2}\|\vx_{t+1} - \vx_{t+1}^{\star}\|_{2}^{2}
    \label{eq:proof-h-smooth}
\end{equation}
Adding \cref{eq:proof-prox-minimizer} \& \cref{eq:proof-h-smooth}, we get an inequality on the difference of $h$ evaluated at two consecutive iterates:
\begin{equation}
    \langle \vm_{t}, \vx_{t+1}^{\star} - \vx_{t}\rangle + \frac{1}{2\alpha_{t}}\|\vx_{t+1}^{\star} - \vx_{t}\|^{2}_{\mD_{t}} - \langle \nabla h(\vx_{t+1}^{\star}), \vx_{t+1} - \vx_{t+1}^{\star}\rangle - \frac{L'}{2}\|\vx_{t+1} - \vx_{t+1}^{\star}\|_{2}^{2} \leq h(\vx_{t}) - h(\vx_{t+1})
    \label{eq:proof-difference-h}
\end{equation}
On the other hand, using the $L$-smoothness of $f$, we can derive the following two inequalities
\begin{align*}
    f(\vx_{t+1}) &\leq f(\vx_{t+1}^{\star}) + \langle \nabla f(\vx_{t+1}^{\star}), \vx_{t+1} - \vx_{t+1}^{\star} \rangle + \frac{L}{2}\|\vx_{t+1} - \vx_{t+1}^{\star}\|_{2}^{2}\\
    f(\vx_{t+1}^{\star}) &\leq f(\vx_{t}) + \langle \nabla f(\vx_{t}), \vx_{t+1}^{\star} - \vx_{t} \rangle + \frac{L}{2}\|\vx_{t+1}^{\star} - \vx_{t}\|_{2}^{2}
\end{align*}
which, once added together, leads to a similar inequality as the one in \cref{eq:proof-difference-h}, involving the difference of $f$ evaluated at two consecutive iterates:
\begin{equation}
    f(\vx_{t+1}) - f(\vx_{t}) \leq \langle \nabla f(\vx_{t+1}^{\star}), \vx_{t+1} - \vx_{t+1}^{\star}\rangle + \langle \nabla f(\vx_{t}), \vx_{t+1}^{\star} - \vx_{t}\rangle + \frac{L}{2} \|\vx_{t+1} - \vx_{t+1}^{\star}\|_{2}^{2} + \frac{L}{2}\|\vx_{t+1}^{\star} - \vx_{t}\|_{2}^{2}
    \label{eq:proof-difference-f}
\end{equation}
Finally, subtracting \cref{eq:proof-difference-f} to \cref{eq:proof-difference-h}, we get the following inequality on the difference of $F$ evaluated at two consecutive iterates
\begin{align}
    \langle \vm_{t} - \nabla f(\vx_{t}), \vx_{t+1}^{\star} - \vx_{t} \rangle - \langle \nabla F(\vx_{t+1}^{\star}), \vx_{t+1} - \vx_{t+1}^{\star} \rangle + \frac{1}{2\alpha_{t}}\|\vx_{t+1}^{\star} - \vx_{t}\|_{\mD_{t}}^{2}\nonumber\\
    \qquad - \frac{L + L'}{2}\|\vx_{t+1} - \vx_{t+1}^{\star}\|_{2}^{2} - \frac{L}{2}\|\vx_{t+1}^{\star} - \vx_{t}\|_{2}^{2} \leq F(\vx_{t}) - F(\vx_{t+1}).\label{eq:proof-difference-F}
\end{align}
\paragraph{Bound on $\|\vx_{t+1}^{\star} - \vx_{t}\|_{2}^{2}$} We will eventually use this inequality to obtain an upper-bound on $\|\vx_{t+1}^{\star} - \vx_{t}\|_{2}^{2}$. But first, we would like to bound some of the remaining terms in this inequality, starting with the first term of the left-hand side of \cref{eq:proof-difference-F}. Since we have assumed in \cref{hyp:prox-approximation} that the series of $\varepsilon_{t+1}^{2}$ was convergent, it guarantees that $\varepsilon_{t+1}$ is bounded by some constant $D'$
\begin{align*}
    \sum_{t=0}^{+\infty} \varepsilon_{t+1}^{2} < +\infty \qquad &\Rightarrow \qquad \forall t \geq 0,\ \|\vx_{t+1} - \vx_{t+1}^{\star}\|_{2} \leq \varepsilon_{t+1} \leq D'\\
    \eqcomment{\cref{hyp:proxgen-assumptions}}{} & \Rightarrow \qquad \forall t \geq 0,\ \|\vx_{t+1}^{\star} - \vx_{t}\|_{2} \leq \|\vx_{t+1} - \vx_{t}\|_{2} + \|\vx_{t+1} - \vx_{t+1}^{\star}\|_{2} \leq D + D'
\end{align*}
As an aside, note that while this provides an inequality over $\|\vx_{t+1}^{\star} - \vx_{t}\|_{2}$, the bound we will eventually derive from \cref{eq:proof-difference-F} will be more appropriate for proving convergence. Using the definition of $\vm_{t} = (1 - \rho_{t})\vg_{t} + \rho_{t}\vm_{t-1}$, together with Young's inequality ($|\langle \va, \vb \rangle| \leq c^{2}/2 \|\va\|_{2}^{2} + 1/2c^{2}\|\vb\|_{2}^{2}$, for any constant $c \neq 0$), we get
\begin{align}
    \big|\langle \vm_{t} - \nabla f(\vx_{t}), \vx_{t+1}^{\star} - \vx_{t}\rangle\big| &= \big|\langle (1 - \rho_{t})\vg_{t} + \rho_{t}\vm_{t-1} - \nabla f(\vx_{t}), \vx_{t+1}^{\star} - \vx_{t}\rangle\big|\nonumber\\[0.4em]
    \eqcomment{Triangular inequality}{} & \leq \big| \langle \vg_{t} - \nabla f(\vx_{t}), \vx_{t+1}^{\star} - \vx_{t}\rangle \big| + \big|\langle \rho_{t} \vg_{t}, \vx_{t+1}^{\star} - \vx_{t}\rangle \big| + \big|\langle \rho_{t}\vm_{t-1}, \vx_{t+1}^{\star} - \vx_{t}\rangle \big|\nonumber\\
    \eqcomment{Young \& CS inequalities}{} & \leq \frac{1}{2L}\|\vg_{t} - \nabla f(\vx_{t})\|_{2}^{2} + \frac{L}{2}\|\vx_{t+1}^{\star} - \vx_{t}\|_{2}^{2} + \rho_{t}(\|\vg_{t}\|_{2} + \|\vm_{t-1}\|_{2})\|\vx_{t+1}^{\star} - \vx_{t}\|_{2}\nonumber\\
    &\leq \frac{1}{2L}\|\vg_{t} - \nabla f(\vx_{t})\|_{2}^{2} + \frac{L}{2}\|\vx_{t+1}^{\star} - \vx_{t}\|_{2}^{2} + 2\rho_{t}(D + D')G.\label{eq:proof-ineq-dotprod}
\end{align}
In particular, the last inequality (\cref{eq:proof-ineq-dotprod}) uses \cref{hyp:proxgen-assumptions} to bound $\|\vg_{t}\|_{2}$, and Lemma 1 of \citep{yun2020proxgen} to bound $\|\vm_{t-1}\|_{2}$, both by $G$. We then use Young's inequality again to bound the second term on the left-hand side of \cref{eq:proof-difference-F}:
\begin{align}
    \big| \langle \nabla F(\vx_{t+1}^{\star}), \vx_{t+1} - \vx_{t+1}^{\star}\rangle \big| &\leq \frac{1}{2C}\|\nabla F(\vx_{t+1}^{\star})\|_{2}^{2} + \frac{C}{2}\|\vx_{t+1} - \vx_{t+1}^{\star}\|_{2}^{2}\nonumber\\[0.4em]
    &\leq \frac{1}{2C}\dist\!\big(\vzero, \widehat{\partial}F(\vx_{t+1}^{\star})\big)^{2} + \frac{C}{2}\|\vx_{t+1} - \vx_{t+1}^{\star}\|_{2}^{2},
    \label{eq:proof-ineq-dotprodF}
\end{align}
where $C$ is a positive constant defined as
\begin{equation*}
    C = \frac{3(2 + 1/\gamma^{2} + 3L + 4L^{2})}{\delta / 2\alpha_{0} - L} > 0.
\end{equation*}
$C$ is indeed positive because we assumed that $\alpha_{0} < \delta/2L$. Finally, since we have assumed in \cref{hyp:proxgen-assumptions} that $\mD_{t} \succeq \delta \mI$ (in other words, $\mD_{t} - \delta \mI$ is positive semi-definite), we can lower-bound $\|\vx_{t+1}^{\star} - \vx_{t}\|_{\mD_{t}}$ by its Euclidean norm
\begin{equation*}
    \frac{1}{2\alpha_{t}}\|\vx_{t+1}^{\star} - \vx_{t}\|_{\mD_{t}}^{2} \geq \frac{\delta}{2\alpha_{t}}\|\vx_{t+1}^{\star} - \vx_{t}\|_{2}^{2} \geq \frac{\delta}{2\alpha_{0}}\|\vx_{t+1}^{\star} - \vx_{t}\|_{2}^{2},
\end{equation*}
where the second inequality is due to the learning rate $\alpha_{t}$ being non-increasing. Putting it all together:
\begin{align*}
    \left(\frac{\delta}{2\alpha_{0}} - \frac{L}{2}\right)\|\vx_{t+1}^{\star} - \vx_{t}\|_{2}^{2} &\leq \frac{1}{2\alpha_{t}}\|\vx_{t+1}^{\star} - \vx_{t}\|_{\mD_{t}}^{2} - \frac{L}{2}\|\vx_{t+1}^{\star} - \vx_{t}\|_{2}^{2}\\[0.4em]
    \eqcomment{\cref{eq:proof-difference-F}}{} &\leq F(\vx_{t}) - F(\vx_{t+1}) - \langle \vm_{t} - \nabla f(\vx_{t}), \vx_{t+1}^{\star} - \vx_{t}\rangle + \langle \nabla F(\vx_{t+1}^{\star}), \vx_{t+1} - \vx_{t+1}^{\star} \rangle\\
    & \qquad\qquad + \frac{L + L'}{2}\|\vx_{t+1} - \vx_{t+1}^{\star}\|_{2}^{2}\\
    \eqcomment{\cref{eq:proof-ineq-dotprod,eq:proof-ineq-dotprodF}}{} &\leq F(\vx_{t}) - F(\vx_{t+1}) + \frac{1}{2L}\|\vg_{t} - \nabla f(\vx_{t})\|_{2}^{2} + \frac{L}{2}\|\vx_{t+1}^{\star} - \vx_{t}\|_{2}^{2} + 2\rho_{t}(D + D')G\\
    & \qquad \qquad + \frac{1}{2C}\dist\!\big(\vzero, \widehat{\partial}F(\vx_{t+1}^{\star})\big)^{2} + \frac{C}{2}\|\vx_{t+1} - \vx_{t+1}^{\star}\|_{2}^{2} + \frac{L + L'}{2}\|\vx_{t+1} - \vx_{t+1}^{\star}\|_{2}^{2}
\end{align*}
Gathering all the terms involving $\|\vx_{t+1}^{\star} - \vx_{t}\|_{2}^{2}$ on the left-hand side of the inequality, and using \cref{hyp:prox-approximation}, we finally get the following upper-bound:
\begin{align}
    \left(\frac{\delta}{2\alpha_{0}} - L\right)\|\vx_{t+1}^{\star} - \vx_{t}\|_{2}^{2} &\leq F(\vx_{t}) - F(\vx_{t+1}) + \frac{1}{2L}\|\vg_{t} - \nabla f(\vx_{t})\|_{2}^{2} + 2\rho_{t}(D + D')G\nonumber\\
    & \qquad \qquad + \frac{1}{2C}\dist\!\big(\vzero, \widehat{\partial}F(\vx_{t+1}^{\star})\big)^{2} + \frac{L + L' + C}{2}\varepsilon_{t+1}^{2}.
    \label{eq:proof-bound-norm}
\end{align}
Note that while we derived the above bound in the case where \cref{hyp:prox-approximation} (2) is satisfied, this inequality is also valid when $\vx_{t+1} = \vx_{t+1}^{\star}$; the terms coming from \cref{eq:proof-ineq-dotprodF} would vanish, making the above inequality looser, but valid for all $t$.

\paragraph{Bound on $\dist\!\big(\vzero, \widehat{\partial}F(\vx_{t+1}^{\star})\big)^{2}$} Recall that by definition of the weighted proximal operator as a minimizer, and by the first-order condition of optimality, we have $\vx_{t+1}^{\star}$ satisfying
\begin{equation*}
    \vzero \in \frac{1}{\alpha_{t}}\mD_{t}(\vx_{t+1}^{\star} - \vx_{t}) + \vm_{t} + \widehat{\partial}h(\vx_{t+1}^{\star}).
\end{equation*}
Adding $\nabla f(\vx_{t+1}^{\star})$, we have \citep{rockafellar2009variational,xu2019analysisprox}
\begin{equation}
    \nabla f(\vx_{t+1}^{\star}) - \vm_{t} - \frac{1}{\alpha_{t}}\mD_{t}(\vx_{t+1}^{\star} - \vx_{t}) \in \nabla f(\vx_{t+1}^{\star}) + \widehat{\partial}h(\vx_{t+1}^{\star}) = \widehat{\partial}F(\vx_{t+1}^{\star}).
\end{equation}
Since we know one element of $\widehat{\partial}F(\vx_{t+1}^{\star})$, we can upper-bound its distance to $\vzero$:
\begin{align}
    \dist\!\big(\vzero, \widehat{\partial}F(\vx_{t+1}^{\star})\big)^{2} &\leq \Big\|\vm_{t} - \nabla f(\vx_{t+1}^{\star}) + \frac{1}{\alpha_{t}}\mD_{t}(\vx_{t+1}^{\star} - \vx_{t})\Big\|_{2}^{2}\nonumber\\
    &= \Big\|\vm_{t} - \nabla f(\vx_{t+1}^{\star}) + (\vx_{t+1}^{\star} - \vx_{t}) + \frac{1}{\alpha_{t}}\mD_{t}(\vx_{t+1}^{\star} - \vx_{t}) - (\vx_{t+1}^{\star} - \vx_{t})\Big\|_{2}^{2}\nonumber\\
    \eqcomment{Young's inequality}{} &\leq 3\left[\|\vm_{t} - \nabla f(\vx_{t+1}^{\star}) + \vx_{t+1}^{\star} - \vx_{t}\|_{2}^{2} + \Big\|\frac{1}{\alpha_{t}}\mD_{t}(\vx_{t+1}^{\star} - \vx_{t})\Big\|_{2}^{2} + \|\vx_{t+1}^{\star} - \vx_{t}\|_{2}^{2}\right]\nonumber\\
    \eqcomment{\cref{hyp:proxgen-assumptions}}{} &\leq 3\Big[\|\vm_{t} - \nabla f(\vx_{t+1}^{\star}) + \vx_{t+1}^{\star} - \vx_{t}\|_{2}^{2} + \frac{1}{\gamma^{2}}\|\vx_{t+1}^{\star} - \vx_{t}\|_{2}^{2} + \|\vx_{t+1}^{\star} - \vx_{t}\|_{2}^{2}\Big]\nonumber\\
    &=3\Big[\|\vm_{t} - \nabla f(\vx_{t+1}^{\star})\|_{2}^{2} + 2 \langle \vm_{t} - \nabla f(\vx_{t+1}^{\star}), \vx_{t+1}^{\star} - \vx_{t}\rangle + \Big(2 + \frac{1}{\gamma^{2}}\Big)\|\vx_{t+1}^{\star} - \vx_{t}\|_{2}^{2}\Big]\label{eq:proof-bound-dist-star1}
\end{align}
Again, we would like to bound some of the remaining terms on the right-hand side of the above inequality. We will start with the first term; using the definition of $\vm_{t}$, we have
\begin{align}
    \|\vm_{t} - \nabla f(\vx_{t+1}^{\star})\|_{2}^{2} &= \|(1 - \rho_{t})\vg_{t} - \nabla f(\vx_{t}) + \nabla f(\vx_{t}) - \nabla f(\vx_{t+1}^{\star}) + \rho_{t}\vm_{t-1}\|_{2}^{2}\nonumber\\[0.4em]
    &\leq 4\Big[\|\vg_{t} - \nabla f(\vx_{t})\|_{2}^{2} + \|\nabla f(\vx_{t+1}^{\star}) - \nabla f(\vx_{t})\|_{2}^{2} + \|\rho_{t}\vg_{t}\|_{2}^{2} + \|\rho_{t}\vm_{t-1}\|_{2}^{2}\Big]\nonumber\\[0.3em]
    \eqcomment{$f$ is $L$-smooth}{} &\leq 4\Big[\|\vg_{t} - \nabla f(\vx_{t})\|_{2}^{2} + L^{2}\|\vx_{t+1}^{\star} - \vx_{t}\|_{2}^{2} + 2\rho_{t}^{2}G^{2}\Big],
\end{align}
where we used the same bounds on $\|\vg_{t}\|_{2}$ and $\|\vm_{t-1}\|_{2}$ as in \cref{eq:proof-ineq-dotprod}. Similarly, we can bound the second term on the right-hand side of \cref{eq:proof-bound-dist-star1}:
\begin{align}
    \big|\langle \vm_{t} - \nabla f(\vx_{t+1}^{\star}), \vx_{t+1}^{\star} - \vx_{t}\rangle \big| & = \big|\langle \vm_{t} - \nabla f(\vx_{t}) + \nabla f(\vx_{t}) - \nabla f(\vx_{t+1}^{\star}), \vx_{t+1}^{\star} - \vx_{t}\rangle \big|\nonumber\\[0.4em]
    \eqcomment{Triangular inequality}{} &\leq \big|\langle \vm_{t} - \nabla f(\vx_{t}), \vx_{t+1}^{\star} - \vx_{t}\rangle \big| + \big| \langle \nabla f(\vx_{t+1}^{\star}) - \nabla f(\vx_{t}), \vx_{t+1}^{\star} - \vx_{t}\rangle \big|\nonumber\\[0.4em]
    \eqcomment{Cauchy-Schwarz inequality}{} &\leq \big|\langle \vm_{t} - \nabla f(\vx_{t}), \vx_{t+1}^{\star} - \vx_{t}\rangle \big| + \|\nabla f(\vx_{t+1}^{\star}) - \nabla f(\vx_{t})\|_{2}\|\vx_{t+1}^{\star} - \vx_{t}\|_{2}\nonumber\\[0.4em]
    \eqcomment{$f$ is $L$-smooth}{} &\leq \big|\langle \vm_{t} - \nabla f(\vx_{t}), \vx_{t+1}^{\star} - \vx_{t}\rangle \big| + L\|\vx_{t+1}^{\star} - \vx_{t}\|_{2}^{2}\nonumber\\
    \eqcomment{\cref{eq:proof-ineq-dotprod}}\qquad & \leq \frac{1}{2L}\|\vg_{t} - \nabla f(\vx_{t})\|_{2}^{2} + \frac{3L}{2}\|\vx_{t+1}^{\star} - \vx_{t}\|_{2}^{2} + 2\rho_{t}(D + D')G\label{eq:proof-ineq-dotprodtp1}
\end{align}
Putting it all together, and using the bound on $\|\vx_{t+1}^{\star} - \vx_{t}\|_{2}^{2}$ we derived earlier:
\begin{align*}
    \dist\!\big(\vzero, \widehat{\partial}F(\vx_{t+1}^{\star})\big)^{2} &\leq 3\Big[\Big(4 + \frac{1}{L}\Big)\|\vg_{t} - \nabla f(\vx_{t})\|_{2}^{2} + \Big(2 + \frac{1}{\gamma^{2}} + 3L + 4L^{2}\Big)\|\vx_{t+1}^{\star} - \vx_{t}\|_{2}^{2}\\
    & \qquad \qquad + 4\rho_{t}(D + D')G + 8\rho_{t}^{2}G^{2}\Big]\\
    \eqcomment{\cref{eq:proof-bound-norm}}{} & \leq C\big(F(\vx_{t}) - F(\vx_{t+1})\big) + \Big[12 + \frac{3}{L} + \frac{C}{2L}\Big]\|\vg_{t} - \nabla f(\vx_{t})\|_{2}^{2} + \frac{1}{2}\dist\!\big(\vzero, \widehat{\partial}F(\vx_{t+1}^{\star})\big)^{2}\\
    &\qquad \qquad + (2C + 4)\rho_{t}(D + D')G + 8\rho_{t}^{2}G^{2} + \frac{C(L + L' + C)}{2}\varepsilon_{t+1}^{2}
\end{align*}
Rearranging the terms together, we get the following bound on $\dist\!\big(\vzero, \widehat{\partial}F(\vx_{t+1}^{\star})\big)^{2}$:
\begin{align}
    \frac{1}{2}\dist\!\big(\vzero, \widehat{\partial}F(\vx_{t+1}^{\star})\big)^{2} & \leq C\big(F(\vx_{t}) - F(\vx_{t+1})\big) + \Big[12 + \frac{3}{L} + \frac{C}{2L}\Big]\|\vg_{t} - \nabla f(\vx_{t})\|_{2}^{2}\nonumber\\
    &\qquad \qquad + (2C + 4)\rho_{t}(D + D')G + 8\rho_{t}^{2}G^{2} + \frac{C(L + L' + C)}{2}\varepsilon_{t+1}^{2}.\label{eq:proof-bound-dist-star}
\end{align}

\paragraph{Bound on $\dist\!\big(\vzero, \widehat{\partial}F(\vx_{t+1})\big)^{2}$} Although we have derived a bound on $\dist\!\big(\vzero, \widehat{\partial}F(\vx_{t+1}^{\star})\big)^{2}$, we are eventually interested in bounding the distance of $\widehat{\partial}F(\vx_{t+1})$ to $\vzero$, at the possible approximation $\vx_{t+1}$ of $\vx_{t+1}^{\star}$. Using the triangular inequality
\begin{align*}
    \dist\!\big(\vzero, \widehat{\partial}F(\vx_{t+1})\big)^{2} &\leq \big[\dist\!\big(\vzero, \widehat{\partial}F(\vx_{t+1}^{\star})\big) + d_{H}\big(\widehat{\partial}F(\vx_{t+1}), \widehat{\partial}F(\vx_{t+1}^{\star})\big)\big]^{2}\\[0.4em]
    \eqcomment{Young's inequality}{} &\leq 2\dist\!\big(\vzero, \widehat{\partial}F(\vx_{t+1}^{\star})\big)^{2} + 2d_{H}\big(\widehat{\partial}F(\vx_{t+1}), \widehat{\partial}F(\vx_{t+1}^{\star})\big)^{2},
\end{align*}
where $d_{H}$ is the Pompeiu–Hausdorff distance between two sets. In the case of \cref{hyp:prox-approximation} (2), where $h$ (and therefore $F$) is differentiable in the $\varepsilon_{t+1}$-ball $\gB$ around $\vx_{t+1}^{\star}$, both of these sets are reduced to a singleton
\begin{align*}
    d_{H}\big(\widehat{\partial}F(\vx_{t+1}), \widehat{\partial}F(\vx_{t+1}^{\star})\big) &= d_{H}\big(\{\nabla F(\vx_{t+1})\}, \{\nabla F(\vx_{t+1}^{\star})\}\big)\\[0.4em]
    & = \|\nabla F(\vx_{t+1}) - \nabla F(\vx_{t+1}^{\star})\|_{2}\\[0.4em]
    \eqcomment{$f$ is $L$-smooth, $h$ is $L'$-smooth in $\gB$}{} &\leq (L + L')\|\vx_{t+1} - \vx_{t+1}^{\star}\|_{2}\\[0.4em]
    \eqcomment{\cref{hyp:prox-approximation}}{} &\leq (L + L')\varepsilon_{t+1}.
\end{align*}
Note that while we derived the inequality above with \cref{hyp:prox-approximation} (2), it is also valid when $\vx_{t+1} = \vx_{t+1}^{\star}$, albeit looser (since the Pompeiu-Hausdorff distance would vanish in the latter case); hence this inequality is valid for all $t$. Using the bound on $\dist\!\big(\vzero, \widehat{\partial}F(\vx_{t+1}^{\star})\big)^{2}$, we can finally get the following bound on $\dist\!\big(\vzero, \widehat{\partial}F(\vx_{t+1})\big)^{2}$
\begin{align}
    \dist\!\big(\vzero, \widehat{\partial}F(\vx_{t+1})\big)^{2} &\leq 4C\big(F(\vx_{t}) - F(\vx_{t+1})\big) + 4\Big[12 + \frac{3}{L} + \frac{C}{2L}\Big]\|\vg_{t} - \nabla f(\vx_{t})\|_{2}^{2}\nonumber\\[0.4em]
    & \qquad \qquad + 4(2C + 4)\rho_{t}(D + D')G + 32\rho_{t}^{2}G^{2} + \big[2C(L + L' + C) + 2(L + L')^{2}\big]\varepsilon_{t+1}^{2}\label{eq:proof-bound-dist}
\end{align}

\paragraph{Convergence result} To prove our final convergence result, we can simply write the expectation over squared distances (where $R$ is uniform) as an average of quantities we have been capable of bounding:
\begin{equation}
    \E_{R}\big[\dist\!\big(\vzero, \widehat{\partial}F(\vx_{R})\big)^{2}\big] = \frac{1}{T}\sum_{t=0}^{T-1}\dist\!\big(\vzero, \widehat{\partial}F(\vx_{t+1})\big)^{2}
    \label{eq:proof-expectation-average}
\end{equation}
Moreover, note that the right-hand side of the inequality in \cref{eq:proof-bound-dist} involves the difference of $F$ at two consecutive iterates, and therefore involves a telescoping series once summed over. Using \cref{hyp:proxgen-assumptions}, we get
\begin{equation*}
    \sum_{t=0}^{T-1}\big(F(\vx_{t}) - F(\vx_{t+1})\big) = F(\vx_{0}) - F(\vx_{T}) \leq F(\vx_{0}) - F(\vx^{\star}) \leq \Delta.
\end{equation*}
We can similarly bound some of the remaining sums. Using the definition of $\rho_{t} = \rho_{0}\mu^{t}$ (\cref{hyp:proxgen-assumptions}), we have
\begin{align*}
    \sum_{t=0}^{T-1}\rho_{t} &\leq \rho_{0}\sum_{t=0}^{+\infty}\mu^{t} = \frac{\rho_{0}}{1-\mu} & \sum_{t=0}^{T-1}\rho_{t}^{2} &\leq \rho_{0}^{2}\sum_{t=0}^{+\infty}\mu^{2t} = \frac{\rho_{0}^{2}}{1 - \mu^{2}},
\end{align*}
and using \cref{hyp:prox-approximation},
\begin{equation*}
    \sum_{t=0}^{T-1}\varepsilon_{t+1}^{2} \leq \sum_{t=0}^{+\infty}\varepsilon_{t+1}^{2} = K.
\end{equation*}
To conclude, using the bound from \cref{eq:proof-bound-dist} inside \cref{eq:proof-expectation-average}, and the various bounds on the sums above, we get the following bound on the expected squared distance:
\begin{equation}
    \E_{R}\big[\dist\!\big(\vzero, \widehat{\partial}F(\vx_{R})\big)^{2}\big] \leq \frac{C_{1}}{T}\sum_{t=0}^{T-1}\|\vg_{t} - \nabla f(\vx_{t})\|_{2}^{2} + \frac{C_{2}\Delta}{T} + \frac{C_{3}}{T},
    \label{eq:proof-convergence-analysis-main-thm}
\end{equation}
where $C_{1}$, $C_{2}$, and $C_{3}$ are three constants independent of $T$, defined as
\begin{align*}
    C_{1} &= 4\Big[12 + \frac{3}{L} + \frac{C}{2L}\Big]\\[0.4em]
    C_{2} &= 4C\\[0.4em]
    C_{3} &= \frac{4(2C + 4)\rho_{0}(D + D')G}{1 - \mu} + \frac{32\rho_{0}^{2}G^{2}}{1 - \mu^{2}} + K\big[2C(L + L' + C) + 2(L + L')^{2}\big].
\end{align*}
\end{proof}

We can also prove \cref{cor:fixed-batch-size} \citep{xu2019analysisprox}, which we recall here:
\begingroup
\def\thecorollary{\ref{cor:fixed-batch-size}}
\begin{corollary}[Fixed mini-batch size]
    If the assumptions of \cref{thm:convergence-analysis} are satisfied, with $T = 2(C_{2}\Delta + C_{3})/\varepsilon^{2}$ and with a fixed mini-batch size $m_{t}$ with $m_{t} = 2C_{1}\sigma^{2}/\varepsilon^{2}$, then the output $\vx_{R}$ of \cref{alg:minimization-composite-loss} satisfies
    \begin{equation*}
        \E_{R}\big[\dist\!\big(\vzero, \widehat{\partial}F(\vx_{R})\big)^{2}\big] \leq \varepsilon^{2},
    \end{equation*}
    where $C_{1}$, $C_{2}$, and $C_{3}$ are the constants from \cref{thm:convergence-analysis}. To have $\E[\dist(\vzero, \widehat{\partial}F(\vx_{R}))] \leq \varepsilon$, it is then sufficient to have $T = O(1 / \varepsilon^{2})$, making the total complexity $O(1 / \varepsilon^{4})$.
\end{corollary}
\addtocounter{theorem}{-1}
\endgroup
\begin{proof}
If we assume that \cref{hyp:proxgen-assumptions} is satisfied, we know that $\vg_{t}$ is estimated using a mini-batch of $m_{t}$ samples of the form $\nabla f(\vx_{t};\xi_{i_{t}})$, then we can further bound the remaining quantity in \cref{eq:proof-convergence-analysis-main-thm}:
\begin{align*}
    \|\vg_{t} - \nabla f(\vx_{t})\|_{2}^{2} &= \E\left[\Big\|\frac{1}{m_{t}}\sum_{i_{t}=1}^{m_{t}}\nabla f(\vx_{t};\xi_{i_{t}}) - \nabla f(\vx_{t})\Big\|_{2}^{2}\right]\\
    &= \frac{1}{m_{t}^{2}}\E\left[\Big\|\sum_{i_{t}=1}^{m_{t}}\big(\nabla f(\vx_{t};\xi_{i_{t}}) - \nabla f(\vx_{t})\big)\Big\|_{2}^{2}\right]\\
    \eqcomment{Jensen's inequality}{}&\leq \frac{1}{m_{t}^{2}}\sum_{i_{t}=1}^{m_{t}}\E\big[\|\nabla f(\vx_{t};\xi_{i_{t}}) - \nabla f(\vx_{t})\|_{2}^{2}\big]\\
    \eqcomment{\cref{hyp:proxgen-assumptions}}{}&\leq \frac{\sigma^{2}}{m_{t}}
\end{align*}
Using \cref{thm:convergence-analysis}, we then have
\begin{equation}
    \E_{R}\big[\dist\!\big(\vzero, \widehat{\partial}F(\vx_{R})\big)^{2}\big] \leq \frac{C_{1}\sigma^{2}}{m_{t}} + \frac{C_{2}\Delta}{T} + \frac{C_{3}}{T}.
    \label{eq:proof-corollary-substitute}
\end{equation}
And finally substituting  $T = 2(C_{2}\Delta + C_{3})/\varepsilon^{2}$ and $m_{t} = 2C_{1}\sigma^{2}/\varepsilon^{2}$ in \cref{eq:proof-corollary-substitute}, we get the expected result
\begin{equation}
    \E_{R}\big[\dist\!\big(\vzero, \widehat{\partial}F(\vx_{R})\big)^{2}\big] \leq \varepsilon^{2}.
\end{equation}
\end{proof}

\newpage
\section{Algorithmic details}
\label{app:algorithmic-details}
In this section, we give some details about the different algorithms used in this work, including the proximal gradient method ProxGen \citep{yun2020proxgen}, the Newton-Raphson algorithm in order to approximate the weighted proximal operator, as well as the procedure used to prune neural networks with indirect sparsity.

In \cref{alg:minimization-loss,alg:minimization-composite-loss}, we show a side-by-side comparison between ProxGen (\cref{alg:minimization-composite-loss}; \citealp{yun2020proxgen}) to minimize a composite loss function, and a standard first-order adaptive method to minimize an objective function $f$. The only difference is highlighted in red, where the weighted proximal operator is applied to the gradient update.
\begin{figure}[ht]
\begin{minipage}[t]{0.46\textwidth}
\begin{algorithm}[H]
  \caption{Minimization of the loss function\\$f(\vx)$ with an adaptive optimizer}
  \label{alg:minimization-loss}
\begin{algorithmic}
    \REQUIRE A loss function $\E_{\xi}[f(\vx;\xi)]$
    \REQUIRE A sequence of learning rates $(\alpha_{t})_{t\geq 0}$, $\vm_{0}$, $\mD_{0}$
    \FOR{$t=0, \ldots, T-1$}
    \STATE Draw a minibatch $\xi_{t}$
    \STATE Update the mean estimate $\vm_{t}$ with $\vg_{t} = \nabla f(\vx_{t};\xi_{t})$
    \STATE Update the preconditioning matrix $\mD_{t}$
    \STATE Update the parameters: $\vx_{t+1} \leftarrow \vx_{t} - \alpha_{t} \mD_{t}^{-1}\vm_{t}$
    \ENDFOR
    \OUTPUT $\vx_{T}$, \OR $\vx_{R}$, $R$ sampled uniformly in $\{1, \ldots, T\}$
\end{algorithmic}
\end{algorithm}
\end{minipage}
\hfill
\begin{minipage}[t]{0.53\textwidth}
\begin{algorithm}[H]
  \caption{Minimization of the composite loss function\\$f(\vx) + h(\vx)$ with an adaptive optimizer \citep{yun2020proxgen}}
  \label{alg:minimization-composite-loss}
\begin{algorithmic}
    \REQUIRE A composite loss function $\E_{\xi}[f(\vx;\xi)] + h(\vx)$
    \REQUIRE A sequence of learning rates $(\alpha_{t})_{t\geq 0}$, $\vm_{0}$, $\mD_{0}$
    \FOR{$t=0, \ldots, T-1$}
    \STATE Draw a minibatch $\xi_{t}$
    \STATE Update the mean estimate $\vm_{t}$ with $\vg_{t} = \nabla f(\vx_{t};\xi_{t})$
    \STATE Update the preconditioning matrix $\mD_{t}$
    \STATE \textcolor{Red}{Update the parameters: $\vx_{t+1} \leftarrow \prox_{\alpha_{t} h}^{\mD_{t}}\!\big(\vx_{t} - \alpha_{t} \mD_{t}^{-1}\vm_{t}\big)$}
    \ENDFOR
    \OUTPUT $\vx_{T}$, \OR $\vx_{R}$, $R$ sampled uniformly in $\{1, \ldots, T\}$
\end{algorithmic}
\end{algorithm}
\end{minipage}
\end{figure}

\subsection{Adaptive optimizers}
\label{app:adaptive-optimizers}
\cref{tab:optimizers} shows, for reference, the updates of the mean estimate $\vm_{t}$ and the preconditioning matrix $\mD_{t}$ for different standard adaptive first-order methods.
\begin{table}[ht]
    \caption{Standard adaptive optimizers used in Deep Learning, with the corresponding updates for the mean estimate $\vm_{t}$, and the preconditioning matrix $\mD_{t}$. Table adapted from \citep{melchior2019adaprox}.}
    \vspace*{0.1in}
    \begin{adjustbox}{center}
    \begin{tabular}{c|cc|ccc}
        \toprule
         & \multicolumn{2}{c|}{Mean estimate} & \multicolumn{2}{c}{Preconditioning matrix}\\
         & $\widehat{\vm}_{t}$ & $\vm_{t}$ & $\vv_{t}$ & $\mD_{t}$\\
        \midrule
        SGD & -- & $\vg_{t}$ & -- & $\mI$\\[0.1em]
        Momentum & \multirow{2}{*}{$\mu\widehat{\vm}_{t-1} + \vg_{t}$} & \multirow{2}{*}{$\widehat{\vm}_{t}$} & \multirow{2}{*}{--} & \multirow{2}{*}{$\mI$}\\
        \citep{rumelhart1986momentum} & & & & \\[0.1em]
        AdaGrad & \multirow{2}{*}{--} & \multirow{2}{*}{$\vg_{t}$} & \multirow{2}{*}{$\vv_{t-1} + \vg_{t}^{2}$} & \multirow{2}{*}{$\sqrt{\vv_{t}} + \varepsilon$}\\
        \citep{duchi2011adagrad} & & & & \\[0.1em]
        RMSprop & \multirow{2}{*}{--} & \multirow{2}{*}{$\vg_{t}$} & \multirow{2}{*}{$\beta \vv_{t-1} + (1 - \beta)\vg_{t}^{2}$} & \multirow{2}{*}{$\sqrt{\vv_{t} + \varepsilon}$}\\
        \citep{tieleman2012rmsprop} & & & & \\[0.1em]
        Adam & \multirow{2}{*}{$\beta_{1}\widehat{\vm}_{t-1} + (1 - \beta_{1})\vg_{t}$} & \multirow{2}{*}{$\dfrac{\widehat{\vm}_{t}}{1 - \beta_{1}^{t}}$} & \multirow{2}{*}{$\beta_{2}\vv_{t-1} + (1 - \beta_{2})\vg_{t}^{2}$} & \multirow{2}{*}{$\sqrt{\dfrac{\vv_{t}}{1 - \beta_{2}^{t}}} + \varepsilon$}\\
        \citep{kingma2014adam} & & & & \\[0.2em]
        \bottomrule
    \end{tabular}
    \end{adjustbox}
    \label{tab:optimizers}
\end{table}

\subsection{Approximation of the weighted proximal operator with Newton-Raphson}
\label{app:newton-raphson}
In \cref{alg:newton-raphson-l2}, we show how to use the Newton-Raphson algorithm as our routine for finding $\theta^{\star}$, and therefore approximate the weighted proximal operator of the $\ell_{2}$ norm. This procedure can be adapted to MCP / $\ell_{2}$ in a straightforward way. Note that while this iterative procedure has to be run at every gradient update during optimization, empirically this induces a reasonable overhead, which can be controlled by the tolerance $\varepsilon$; lower tolerance yields fewer iterations per gradient update. See \cref{sec:approximation-weighted-prox} for empirical evidence.
\begin{figure}[ht]
\begin{minipage}{0.6\textwidth}
\begin{algorithm}[H]
  \caption{Approximation of the weighted proximal operator of the $\ell_{2}$ norm ($h(\vx) = \lambda \|\vx\|_{2}$) with Newton-Raphson}
  \label{alg:newton-raphson-l2}
\begin{algorithmic}
    \REQUIRE A point $\vx$ such that $\|\mD\vx\|_{2} > \alpha\lambda$.
    \REQUIRE An error tolerance $\varepsilon$
    \STATE Initialization: $\theta \leftarrow (\|\mD\vx\|_{2} - \alpha\lambda) / d_{\max}$ \hfill $\triangleright$ \cref{prop:bounds-theta-weighted-prox-l2}
    \WHILE{$|G(\theta)| > \varepsilon$}
        \STATE $\theta \leftarrow \theta - \dfrac{G(\theta)}{G'(\theta)}$
    \ENDWHILE
    \OUTPUT $\big[\prox_{\alpha h}^{\mD}(\vx)\big]_{i} = \dfrac{d_{i}\theta x_{i}}{d_{i}\theta + \alpha \lambda}$ \hfill $\triangleright$ \cref{thm:weighted-prox-l2}
\end{algorithmic}
\end{algorithm}
\end{minipage}\hfill%
\begin{minipage}{0.35\textwidth}
    Where $G$ is the function defined by
    \begin{equation*}
        G(\theta) = \sum_{i=1}^{n}\left[\frac{d_{i}x_{i}}{d_{i}\theta + \alpha \lambda}\right]^{2} - 1
    \end{equation*}
\end{minipage}
\end{figure}

\subsection{Pruning with indirect sparsity}
\label{app:pruning-indirect-sparsity}
In \cref{sec:pruning-indirect-sparsity}, we argued that structured sparsity had benefit not only on the layer where groups of variables were zeroed-out, but on neighboring layer as well. \cref{fig:indirect-sparsity} shows an illustration of this behaviour on a convolutional neural network, inspired by \cref{fig:structured-sparsity}, where structured sparsity has been applied channel-wise.
\begin{figure}[ht]
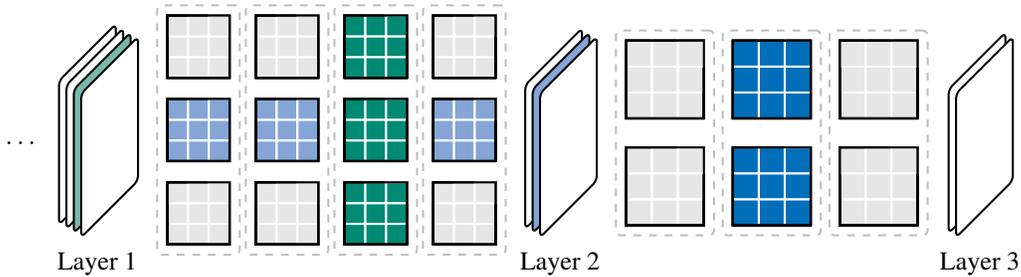

    \centering
    \includestandalone[width=0.8\linewidth]{figures/indirect-sparsity-pruning}
    \caption{Pruning a convolutional neural network with indirect sparsity. This illustration shows the groups of filters set to zero thanks to the structured sparsity inducing penalty (green for the first layer, blue for the second). The filters that can be further pruned in the first layer thanks to indirect sparsity are shown in light-blue.}
    \label{fig:indirect-sparsity}
\end{figure}

The algorithm to prune these filters operates from right to left. When a group of filters are set to zero thanks to the structured sparsity inducing penalty, the corresponding channel at the previous layer (e.g. second channel of Layer 2 in \cref{fig:indirect-sparsity}, shown in light-blue) does not contribute to the predictions of the network anymore; here, the representation at Layer 3 is independent of the channel in light-blue. Therefore the channel in light-blue can be ignored from prior computations as well, meaning that the filters in light-blue can be set to zero as well, without functionally affecting the neural network: these filters were pruned \emph{indirectly} by the structure in the sparsity patterns in Layer 2.

On top of this, the structured sparsity inducing penalty applied to the first layer can also set some groups of variables to zero (here the group shown in green); this procedure can then be run from the output layer, all the way to the input layer. This leads to significantly fewer parameters: in \cref{fig:indirect-sparsity}, the first layer only has 50\% of non-zero parameters left, while only 25\% were pruned directly through the structured sparsity inducing penalty applied to that layer (i.e. 75\% group sparsity). Finally, note that while we showed the effect of indirect sparsity on a convolutional neural network, this applies similarly to linear layers with row-wise groups.

\section{Experimental details}
\label{app:experimental-details}
In this section, we provide details details about the experiments in \cref{sec:experimental-results}, as well as additional results on Residual Networks.

\subsection{Convolutional Neural Networks}
\label{app:convolutional-neural-networks}
In order to accelerate training, the VGG-16 network used in our experiments only has 14 layers, as opposed to 16 \citep{simonyan2015vgg}. It has the same feature extraction body as \citet{simonyan2015vgg} (with 13 convolutional layers), but only has a single linear layer for the classification part (as opposed to 3 layers). We trained this network on CIFAR-10 for 100 epochs using Adam \citep{kingma2014adam} with decoupled weight-decay, with a learning rate $\alpha = 10^{-3}$ (decreasing by a factor of $10$ every $40$ epochs), and a weight-decay parameter of $\lambda_{wd} = 5 \times 10^{-3}$.

For our experiments with structured sparsity inducing penalties (both trained using subgradient methods and proximal gradient methods), we used the following hyperparameters:
\begin{table}[h]
    \centering
    \begin{tabular}{lcc}
        \toprule
        Penalty & $\lambda$ & $\beta$\\
        \midrule
        $\ell_{1}/\ell_{2}$ & $2 \times 10^{-5}$ & -- \\
        Group MCP & $2 \times 10^{-5}$ & $5 \times 10^{3}$ \\
        \bottomrule
    \end{tabular}
    \label{tab:cnn-hyperparameters}
\end{table}

Note that in all cases, $\lambda_{g}$ was reweighted by the size of the groups with $\lambda_{g} = \lambda \sqrt{|g|}$ \citep{murphy2012mlapp}.

\begin{figure}[t]
    \centering
    \includegraphics[width=0.45\linewidth]{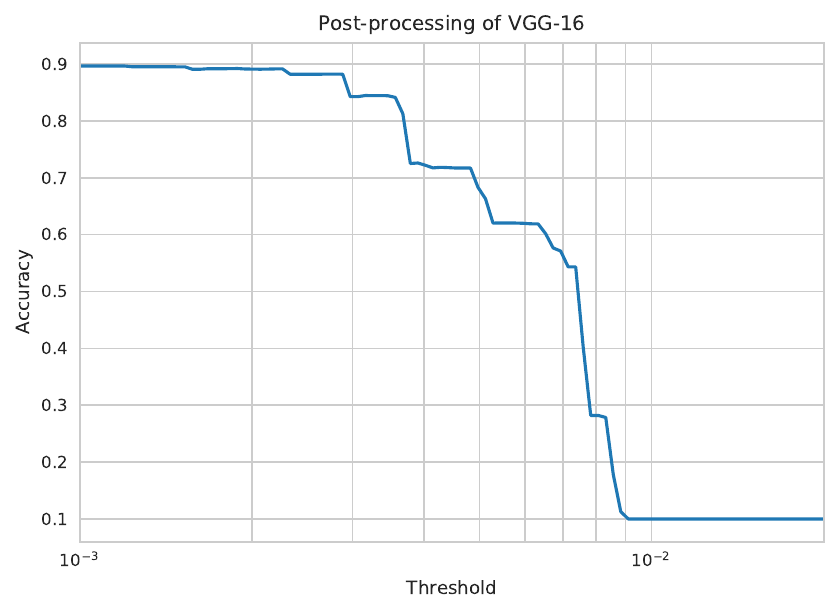}%
    \includegraphics[width=0.45\linewidth]{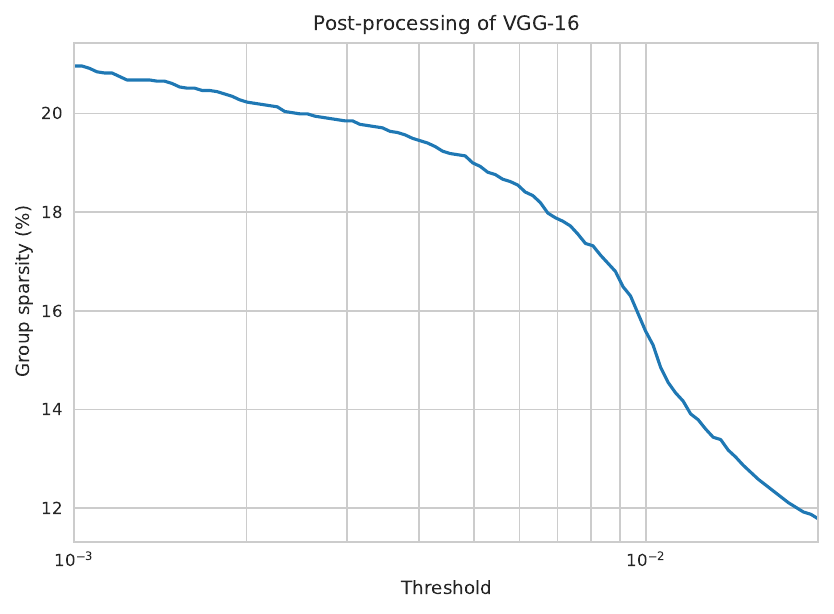}%
    \vspace*{-1em}
    \caption{Post-processing of VGG-16, trained with subgradient methods \citep{wen2016structured}, with the mixed $\ell_{1}/\ell_{2}$ norm. Both graphs show the performance in terms of test accuracy (left) and group sparsity (right) as the threshold used for pruning varies.}
    \label{fig:vgg16-threshold}
\end{figure}

\paragraph{Post-processing} As noted in \cref{sec:convolutional-neural-networks}, and confirming the observations from \citep{bach2011optimization}, the solutions found by training VGG-16 with subgradient methods were not sparse (i.e. no group of variables was set to zero), despite the structured sparsity inducing penalties added. Following \citet{wen2016structured}, we applied a post-processing step after training in order to prune out parameters with small values. Moreover since we are interested in structured sparsity, we applied the thresholding step at the level of groups. More precisely, give a threshold $t$, we pruned out groups of variables whose $\ell_{2}$ norm was smaller than this threshold (rescaled by the group sizes):
\begin{equation*}
    \|\vx_{g}\|_{2} \leq t\sqrt{|g|} \qquad \Rightarrow \qquad \vx_{g} = \vzero.
\end{equation*}
\cref{fig:vgg16-threshold} shows the impact of this post-processing step as $t$ varies. We can observe that the accuracy is highly sensitive to the choice of the threshold $t$. The values reported in \cref{tab:vgg16-group-sparsity} use a threshold $t = 2.5\times 10^{-3}$ to balance group sparsity with test accuracy. Finally, recall that proximal gradient methods did not require any post-processing step, since they are capable of returning sparse solutions directly where whole groups of variables are set to zero.


\subsection{Residual Networks}
\label{app:residual-networks}
We also applied our inexact proximal gradient method with both structured sparsity inducing penalties on a ResNet-34 network, trained on CIFAR-10. Similar to VGG-16 in \cref{sec:convolutional-neural-networks}, we used channel-wise and row-wise groups, depending on the nature of the layer (convolutional and linear layers respectively). Overall, this corresponds to 8k groups, in a model containing 21M parameters. Contrary to our experiments with VGG-16, we only trained the networks using proximal gradient methods, and not subgradient methods.

We trained the network for 200 epochs using Adam with decoupled weight decay, with a learning rate $\alpha = 10^{-3}$ (decreasing by a factor 10 after 150 epochs), and a weight decay parameter of $\lambda_{wd} = 5\times 10^{-4}$. Performance in terms of group sparsity and accuracy are reported in \cref{tab:resnet-group-sparsity}. The constants $\lambda_{g}$ are reweighted by the size of the groups with $\lambda_{g} = \lambda\sqrt{|g|}$.

\begin{table}[h]
    \centering
    \caption{Performance of ResNet-34 trained on CIFAR-10, with different structured sparsity inducing penalties. Here, group sparsity is the proportion of groups (out of 8k) with non-zero norm.}
    \vspace*{0.1in}
    \begin{tabular}{lcc|cc}
        \toprule
        Penalty & $\lambda$ & $\beta$ & Group Sparsity & Test accuracy\\
        \midrule
        Baseline & -- & -- & $44.90\%$ & $93.39\%$\\
        $\ell_{1}/\ell_{2}$ & $1 \times 10^{-6}$ & -- & $32.64\%$ & $93.11\%$ \\
        Group MCP & $1 \times 10^{-5}$ & $1 \times 10^{3}$ & $31.04\%$ & $93.45\%$ \\
        \bottomrule
    \end{tabular}
    \label{tab:resnet-group-sparsity}
\end{table}

Similar to our results with VGG-16, we can observe that both lead to high levels of groups sparsity (i.e. few groups are non-zero), with minimal impact on the final test accuracy. Surprisingly, we also observe that the baseline model trained with no additional structured sparsity inducing penalty leads to some groups being zeroed-out. This is an interesting phenomenon, which is probably caused by the combination of weight decay together with skip-connections.

\begin{figure}[ht]
    \centering
    \includegraphics[width=0.9\linewidth]{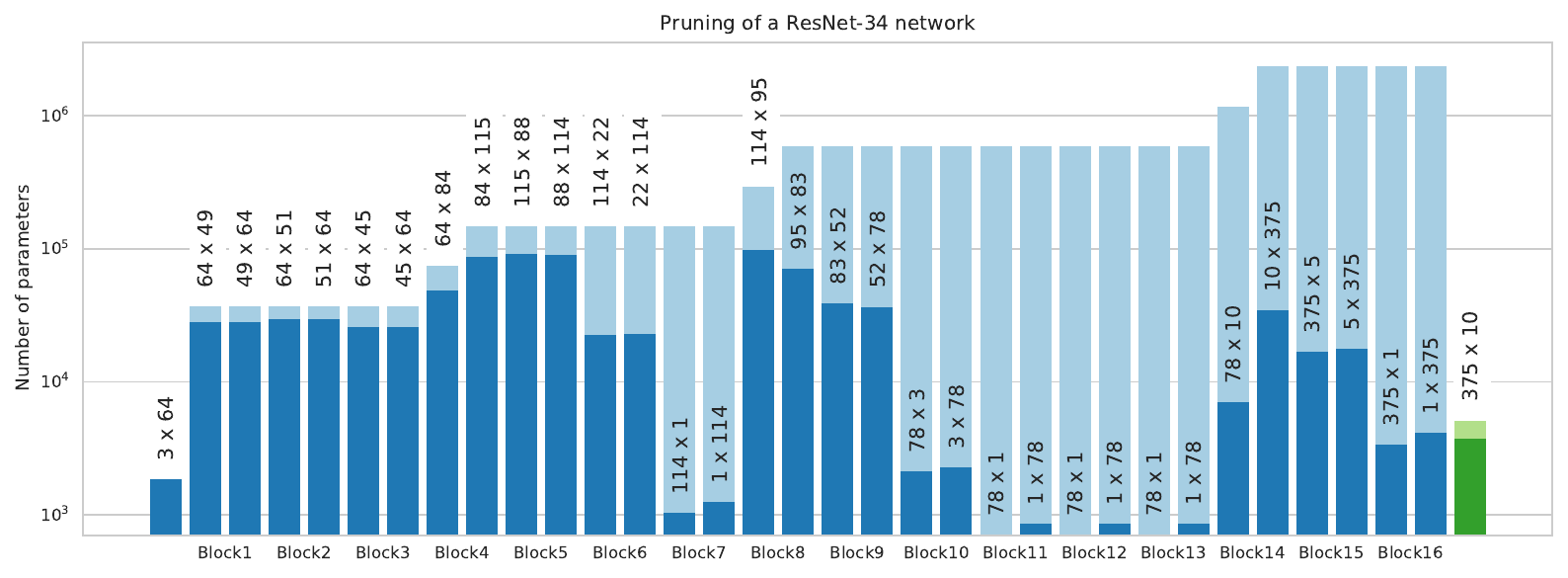}
    \caption{Comparison of the size of the network before (light) and after (dark) pruning, for each layer of ResNet-34. Comvolutional layers (with their corresponding batch-normalization) are represented in blue, and the linear layer in green. The layers are grouped in blocks with two convolutional layers, and an additional skip-connection (not shown here). The label above each bar represents the size of the layer after pruning.}
    \label{fig:pruned-resnet34}
\end{figure}

Similar to \cref{sec:pruning-indirect-sparsity}, we can also prune the ResNet-34 even more using indirect sparsity, thanks to our choice of row-wise groups in the structured sparsity inducing penalties. \cref{fig:pruned-resnet34} shows the effect of pruning on the different layers of the network, for the ResNet-34 trained with the group MCP penalty (reaching a group sparsity of $31.04\%$ in \cref{tab:resnet-group-sparsity}). Interestingly, some convolutional blocks get almost entirely bypassed by their skip connections (e.g. Block7, Blocks11-13, and Block16), having an effect similar to \emph{depth-wise} sparsity \citep{wen2016structured}, without explicitly using depth-wise groups. The total effective sparsity here is $4\%$, corresponding to 919k parameters in the pruned model.

\subsection{Large-scale Transformers}
\label{app:large-scale-transformers}

In all our experiments, the BERT networks were fine-tuned on SQuAD 1.1 for 3 epochs using Adam with linearly decreasing learning rate, starting at $\alpha = 5 \times 10^{-5}$. The networks reported in \cref{tab:squad-group-sparsity} were obtained using multiple values of the hyperparameters $\lambda$ and $\beta$, which are given here:

\begin{table}[h]
    \centering
    \begin{tabular}{lcc|ccc}
        \toprule
        & $\lambda$ & $\beta$ & Group Sparsity & EM & F1\\
        \midrule
        Baseline & -- & -- & -- & $81.01$ & $88.27$ \\
        \midrule
        \multirow{2}{*}{$\ell_{1}/\ell_{2}$} & $1 \times 10^{-5}$ & -- & $84.34\%$ & $72.29$ & $81.96$ \\
        & $2 \times 10^{-5}$ & -- & $76.14\%$ & $66.16$ & $77.50$ \\[0.2em]
        \multirow{2}{*}{Group MCP} & $5 \times 10^{-5}$ & $8 \times 10^{2}$ & $81.49\%$ & $75.40$ & $84.50$ \\
        & $1 \times 10^{-4}$ & $4 \times 10^{2}$ & $61.33\%$ & $69.47$ & $79.95$ \\
        \bottomrule
    \end{tabular}
    \label{tab:squad-hyperparameters}
\end{table}

Again, $\lambda_{g}$ was reweighted by the size of the groups with $\lambda_{g} = \lambda\sqrt{|g|}$.

\paragraph{Choice of groups} As mentioned in \cref{sec:transformers}, in our experiments we used \emph{row-wise} structured sparsity to encourage entire rows of every weight matrix in BERT to be set to zero; this represents 86k groups in total. This choice was motivated by finding a network, where intermediate representations are only influenced by a subset of dimensions from the previous layer. This structure also helps in pruning, as shown in \cref{sec:pruning-indirect-sparsity} and \cref{app:pruning-indirect-sparsity}.

In the context of Transformers, we could use the structure of the network itself even more in order to better define our groups. For example, with our current definition of row-wise groups, the 3 matrices to define the keys, queries and values in a self-attention layer see their rows being penalized independently from one another (i.e. the groups are define as the rows for each of these 3 matrices). However we could also encourage the keys, queries, and values to use the same information from (a subset of) the previous layer. Concretely, this is made possible by grouping together the rows from these 3 matrices from the self-attention layer (effectively reducing by 3x the number of groups for this layer). This would have a more significant impact in terms of pruning from indirect sparsity, since the key, query, and value modules would share the same subset of input dimensions (as opposed to what we currently do, where we can only prune indirectly one layer based on the intersection of the groups set to zero in the 3 matrices of the self-attention layer).

We can further improve this by combining the effect of shared and individual groups for the weights of self-attention layers. Indeed, while we restricted our attention in this paper on disjoint groups, it is also possible to encourage structured sparsity based on a groups following a \emph{tree-structured hierarchy} \citep{bach2011optimization}. Here, this would mean that we can encourage a subset of input dimensions based on groups spanning the 3 weight matrices of the self-attention layer, as well as having ``specialized'' dimensions for the keys, queries, and values with groups corresponding to the rows of each individual weight matrix (which are disjoint subgroups from those spanning all 3 matrices). The application of structured sparsity inducing penalties on tree-structured groups with adaptive proximal gradient methods is left as future work.

\end{document}